\documentclass[letterpaper]{article}
\usepackage[preprint]{aaai2027}
\usepackage[hyphens]{url}
\usepackage{graphicx}
\usepackage{caption}
\usepackage{subcaption}
\usepackage{float}
\usepackage{mathtools}
\usepackage{array}
\usepackage{tabularx}
\usepackage{natbib}
\usepackage{amsmath,amssymb,amsthm}
\usepackage{booktabs}
\usepackage{algorithm}
\usepackage{algpseudocode}

\newtheorem{theorem}{Theorem}
\newtheorem{proposition}{Proposition}
\newtheorem{definition}{Definition}
\newtheorem{corollary}{Corollary}
\theoremstyle{definition}
\newtheorem{assumption}{Assumption}
\theoremstyle{plain}
\newtheorem{remark}{Remark}
\newcommand{\R}{\mathbb{R}}
\newcommand{\cD}{\mathcal{D}}
\newcommand{\cX}{\mathcal{X}}

\newcommand{\cZ}{\mathcal{Z}}
\newcommand{\cE}{\mathcal{E}}
\newcommand{\cC}{\mathfrak{C}}
\newcommand{\cA}{\mathsf{A}}
\newcommand{\abstain}{\bot}
\newcommand{\abstainF}{\bot_{\mathrm F}}
\newcommand{\abstainD}{\bot_{\mathrm D}}
\newcommand{\abstainS}{\bot_{\mathrm S}}
\newcommand{\diam}{\operatorname{diam}}
\newcommand{\supp}{\operatorname{supp}}
\newcolumntype{Y}{>{\raggedright\arraybackslash}X}
\title{Certifying Plans under Model Mismatch: A Trilemma for Reachability \\ from Scarce Data}
\author{
    Yanliang Huang\textsuperscript{\rm 1,*},
    Zhen Zhang\textsuperscript{\rm 1,*},
    Ahmad Hafez\textsuperscript{\rm 1},
    Wenyuan Wu\textsuperscript{\rm 1},
    Peng Xie\textsuperscript{\rm 1},
    Zhuoqi Zeng\textsuperscript{\rm 2},
    Amr Alanwar\textsuperscript{\rm 1}
}
\affiliations{
    \textsuperscript{\rm 1}School of Computation, Information and Technology, Technical University of Munich, Munich, Germany\\
    \textsuperscript{\rm 2}School of Engineering, Hainan Bielefeld University of Applied Sciences, Hainan, China\\
    \{yanliang.huang, zhenzhang.zhang, a.hafez, wenyuan.wu, p.xie, alanwar\}@tum.de, zhuoqi.zeng@hainan-biuh.edu.cn\\
    \textsuperscript{*}These authors contributed equally.
}

\begin{document}

\maketitle

\begin{abstract}
Sim-to-real policies are designed under nominal dynamics, but target-system
trials may yield only a few isolated one-step transitions. We study
pre-execution certification of a fixed control sequence, such as an action chunk
produced by a learned policy. If the sequence reaches an unobserved state-input
region, the observations remain consistent with target systems
whose trajectories separate along it by an arbitrarily large amount. Any deterministic certifier
sound for all of them must then decline to certify or return a reachable tube
with arbitrarily large projected width. For bounded smooth classes of the
target--nominal model error, we derive a finite plan-dependent projected-width lower bound.
These results expose a trilemma among uniform
trajectory containment, finite projected width, and unrestricted model-error
behavior beyond the observations. ForeReach requires a supplied componentwise Lipschitz bound on the model
error. Observed transition pairs can refute this declaration but cannot
establish it outside the observed locations. Conditional on a valid
declaration, our method constructs a set-membership envelope for the model error, propagates a
zonotopic reachable tube, and certifies only when propagation remains within
the certification domain and every projected tube slice avoids the unsafe set.
In two benchmark systems, calibration baselines may remain narrow after losing
trajectory containment outside data support, whereas our method declines to
certify unsupported sequences and recovers certification when relevant target
data and sufficient obstacle clearance are available.
\end{abstract}

\section{Introduction}
\label{sec:introduction}

Sim-to-real transfer is a central challenge for deploying learned policies on
robots. The policies are trained in simulation, whose dynamics differ from the real
system's. This model error between target and nominal dynamics can make behavior
safe in simulation unsafe on hardware \citep{knuth2021learneddynamics,srinivasan2026cpslsmpc}. A
candidate control sequence, therefore, needs certification before it is applied to
the target system. A nominal model, for example, a simulator, is available, but
the target data consist of only a small set of isolated one-step transitions. This
setting fits policies that emit finite action chunks, such as ACT, Diffusion
Policy, and $\pi_0$ \citep{zhao2023act,chi2023diffusionpolicy,black2025pi0}. Whether the task is
closed-loop or open-loop, the executed chunk is a fixed control sequence, and we study its open-loop
pre-execution certification, for which
a reachable tube tests the sequence for collision
\citep{michaux2024sparrows,kwon2024crows}. The central difficulty is
that the nominal rollout may enter a state-input region with no target
observation.

Each observed transition reveals the model error only at its sampled pair. Away from the data, this error can change inside a small neighborhood without altering any observation, so two smooth target systems fit the same dataset yet produce different next states once the fixed sequence enters that neighborhood, after which the dynamics amplify their separation. A certifier sound for both must decline or return a tube wide enough to contain both trajectories. Without a bound on how rapidly the error can vary, the required width has no finite uniform bound, and bounded smoothness makes it finite but still governed by how the sequence propagates the local difference.

We propose ForeReach, a pre-execution certifier for a fixed control sequence. It
requires, as side information, a componentwise Lipschitz bound on the
target--nominal model error over a prescribed certification domain. This bound
must be established independently of the sparse transition data, for example from
analytic dynamics with bounded parameters or a separately certified residual model.
The observed transitions cannot verify the bound outside the sampled locations, and
can only reveal contradictions between the bound and observed residual pairs.
Conditional on the bound holding throughout the certification domain, our method
uses all observations to bound the model error over each queried set and propagates
the resulting uncertainty with the nominal model as a zonotopic reachable tube. It issues a certificate only when the propagated
tube stays inside that domain while every projected slice stays clear of the unsafe set.
Figure~\ref{fig:pipeline} summarizes the procedure.

We make the following contributions.
\begin{itemize}
\item We prove that unrestricted model-error behavior beyond the observations
forces any uniformly sound deterministic certifier to decline or return a tube
with arbitrarily large projected width, and we derive a finite plan-dependent
lower bound for bounded smooth classes of the model error.
\item Our method combines a supplied componentwise model-error bound, a pairwise
consistency test, set-membership envelopes for the model error, and zonotopic
propagation to provide trajectory containment conditional on the supplied
declaration and to certify only after the certification-domain and unsafe-set
checks succeed.
\item Experiments on two benchmark systems, including fixed action chunks produced by
a learned policy, distinguish trajectory containment from plan certification, which
additionally requires the tube to prove obstacle avoidance, and show how relevant
target observations recover certification and expose failure under incorrect
Lipschitz declarations.
\end{itemize}

\begin{figure*}[t]
  \centering
  \includegraphics[width=\textwidth]{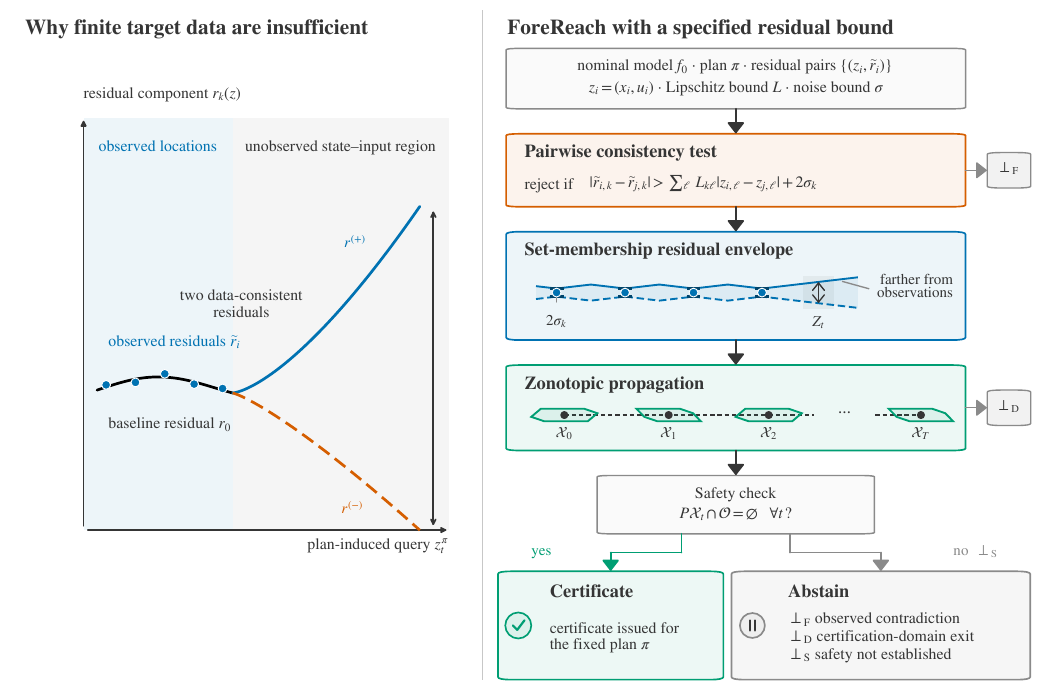}
  \caption{Finite observations admit data-consistent residuals that separate at
  an unobserved plan query. With a supplied Lipschitz bound, our method checks
  observed contradictions, propagates a set-membership residual envelope, and
  returns a certificate or a cause-labeled abstention.}
  \label{fig:pipeline}
\end{figure*}

\section{Preliminaries and Problem Setup}
\label{sec:problem}

\subsection{Preliminaries}

\paragraph{Set operations.}
For sets $A,B$ and a matrix $M$, we write $A\oplus B:=\{a+b:a\in A,\ b\in B\}$ for
the Minkowski sum, $MA:=\{Ma:a\in A\}$ for the linear image, and $A\times B$ for the
Cartesian product.

\paragraph{Zonotopes \citep{conf:zono1998}.}
A zonotope
$Z\subseteq\R^d$ with center $c_Z\in\R^d$ and generator matrix
$G_Z=[g_Z^{(1)}\ \cdots\ g_Z^{(\gamma_Z)}]\in\R^{d\times\gamma_Z}$ is
\begin{equation}
  Z=\langle c_Z,G_Z\rangle
  :=\{c_Z+G_Z\beta:\|\beta\|_\infty\leq1\}.
  \label{eq:zonotope-definition}
\end{equation}
Its order is $\gamma_Z/d$. Linear maps act on the center and generators, Minkowski
sums concatenate generator matrices, Cartesian products stack centers with
block-diagonal generators, and the singleton $\{u\}$ is the degenerate zonotope
$\langle u,0\rangle$.

\subsection{Problem Setup}

\paragraph{Dynamics and observations.}
Let $\mathsf{X}\subseteq\R^n$ and $\mathsf{U}\subseteq\R^m$ be the state and
input domains. At time $t$, the state is $x_t\in\mathsf{X}$ and the applied input
is $u_t\in\mathsf{U}$, and $z_t=(x_t,u_t)$ is the state-input pair in the joint
domain $\cZ:=\mathsf{X}\times\mathsf{U}\subseteq\R^{d_z}$ with $d_z=n+m$. Throughout, $D_x$ and $D_z$ denote Jacobians with respect to $x$ and $z$,
respectively. The known nominal model $f_0:\cZ\to\R^n$ comes from physics, a simulator, or an existing platform model. The unknown target model is
$f_\star:\cZ\to\R^n$, and their difference is the residual $r_\star:=f_\star-f_0$,
the target--nominal model error, so that
\begin{equation}
  x_{t+1}=f_\star(z_t)=f_0(z_t)+r_\star(z_t).
  \label{eq:dynamics}
\end{equation}
The available target data are $N$ isolated one-step transitions
with a supplied componentwise observation-error bound $\sigma\in\R_{\geq0}^n$,
\begin{equation}
  \begin{aligned}
    \cD_N&=\{(z_i,x_i^+)\}_{i=1}^{N},
      \qquad z_i\in\cZ,\\
    x_i^+&=f_\star(z_i)+\eta_i,
      \qquad |\eta_{i,k}|\leq\sigma_k,\\
    \widetilde r_i
      &:=x_i^+-f_0(z_i)
       =r_\star(z_i)+\eta_i .
  \end{aligned}
  \label{eq:data}
\end{equation}
Here $x_i^+$ is the observed successor of sample $z_i$, the vector
$\eta_i\in\R^n$ is the unknown observation error bounded componentwise by
$\sigma$, and the sample locations $z_i$ are treated as exact. No
magnitude or smoothness bound on $r_\star$ is imposed at this stage. Throughout,
$i,j\in\{1,\ldots,N\}$ index observations, $k\in\{1,\ldots,n\}$ indexes residual
coordinates, and $\ell\in\{1,\ldots,d_z\}$ indexes state-input coordinates, and the
target dynamics are deterministic so the observation error is the only source of
uncertainty.

\paragraph{Certification task and guarantees.}
Let $T\geq1$ be the finite certification horizon. The candidate plan is the
open-loop control sequence $\pi=(u_0,\ldots,u_{T-1})\in\mathsf{U}^T$. The certification
task $\mathcal Q=(\cX_0,P,\mathcal O,T)$ specifies the initial set
$\cX_0\subseteq\mathsf{X}$, a nonzero linear safety projection $P:\R^n\to\R^p$ that
maps the full state to the task-relevant safety coordinates, the unsafe set
$\mathcal O\subseteq\R^p$, and the horizon $T$, with $PS=\{Px:x\in S\}$ for any
state set $S$. A returned reachable tube is $\mathbf X=(\cX_0,\ldots,\cX_T)$. For
a residual $r:\cZ\to\R^n$ and an initial state $x_0\in\cX_0$, let $x_t^r$ denote
the trajectory generated by $f_0+r$ under $\pi$, with the dependence on $x_0$
suppressed in the notation. The tube is valid for $r$ if
$x_t^r\in\cX_t$ for every $x_0\in\cX_0$ and every $t=0,\ldots,T$. It is
uniformly sound over a residual class if it is valid for every residual in that class.

A safe certificate requires $P\cX_t\cap\mathcal O=\varnothing$ for every
$t\leq T$. A certifier may instead decline to certify, written $\abstain$, and
our method records the cause of this outcome. For any
state set $S$, write
$\diam_P(S)=\sup_{x,y\in S}\|P(x-y)\|_2$ and
$\rho_P(S)=\diam_P(S)/2$. Because uniform soundness alone permits an uninformative
tube that is valid only by being arbitrarily wide, we also bound the projected
half-width. Given a task-relevant threshold $0\leq\bar\rho<\infty$, a returned tube
is $\bar\rho$-informative when the certifier returns that tube and
$\max_{0\leq t\leq T}\rho_P(\cX_t)\leq\bar\rho$. When the projected slice $P\cX_t$ is
centrally symmetric, let $c_t$ denote its center and define the centered deviation
set $E_t:=\{y-c_t:y\in P\cX_t\}$, so that $P\cX_t=c_t\oplus E_t$ with $E_t=-E_t$.
Define
\begin{equation}
  R_t:=\sup_{e\in E_t}\|e\|_2,
  \qquad
  m_t:=\operatorname{dist}(c_t,\mathcal O).
  \label{eq:safety-radius-margin}
\end{equation}
Here $R_t$ is the tube radius in the safety coordinates, $m_t$ the clearance to the
unsafe set, and $R_t<m_t$ is sufficient for safety. For the centrally symmetric
zonotope slices used here, $R_t=\rho_P(\cX_t)$.

\section{Information Limits from Finite One-Step Data}
\label{sec:theory}

A certifier that must return a sound tube for every target system consistent
with the one-step data faces three properties that cannot hold together. The
returned tube can be uniformly sound, so that it contains the trajectory of
every consistent target system. Its projected width can stay finite, so that the
certificate carries task-relevant information. The model error can stay
unrestricted away from the observed locations. When the plan reaches a query with no
nearby target data, any deterministic certifier attains at most two of these
three, and we call this three-way exclusion the trilemma. Theorem~\ref{thm:scarce-data} proves one form, where an unrestricted
model error forces infinite width, and Theorem~\ref{thm:propagated-width} refines
it, where bounded smoothness still forces a positive plan-dependent width.

Fix the certification task $\mathcal Q$. An admissible deterministic certifier
$\cA$ maps $(f_0,\cD_N,\sigma,\pi)$ to a reachable tube or declines to certify.

Uniform soundness requires the tube to be valid for every residual $r:\cZ\to\R^n$
consistent with the observations and their noise bounds. Define the consistency
class
\begin{equation}
  \begin{split}
    \cC(\cD_N,\sigma):=\{\, r:\ &|\widetilde r_i-r(z_i)|\leq\sigma\\
    &\text{componentwise for every }i \,\}.
  \end{split}
  \label{eq:consistency}
\end{equation}
An admissible certifier is unrestricted beyond the observations if its
soundness claim ranges over all $r\in\cC(\cD_N,\sigma)$ without imposing any
additional restriction on $r$ away from the observed locations
$\{z_i\}_{i=1}^N$.

The following indistinguishable-pair construction drives both results. Choose a
smooth residual $r_0\in\cC(\cD_N,\sigma)$ and an initial state $x_0^0\in\cX_0$. Let $F_b=f_0+r_0$, and define the baseline rollout and queries
by $x_{t+1}^0=F_b(x_t^0,u_t)$ and $z_t^0=(x_t^0,u_t)$. Suppose an open ball
$B_\epsilon(z_\tau^0)\subseteq\cZ$ around one baseline query contains no sample
location $z_i$ and no other baseline query $z_s^0$ for $s\neq\tau$. Let $\phi$ be
a smooth perturbation that is zero outside this data-free neighborhood and
satisfies $\phi(z_\tau^0)=1$, so it changes the dynamics only near this single
unobserved query. For a coordinate $k$, amplitude $a>0$, and the $k$-th
standard basis vector $e_k\in\R^n$, define
\[
  r^{(\pm)}=r_0\pm a\phi e_k.
\]
Both residuals equal $r_0$ at every sampled location and share the noise
realization $\eta_i=\widetilde r_i-r_0(z_i)$, so they produce the same data and input to $\cA$.

Let $x_t^{(\pm)}$ be the rollouts from $x_0^0$ under $f_0+r^{(\pm)}$ and $\pi$,
whose trajectories coincide through time $\tau$ and satisfy
\[
  x_{\tau+1}^{(+)}-x_{\tau+1}^{(-)}=2ae_k.
\]
Any projected tube slice containing both states must therefore have half-width at
least $a\|Pe_k\|_2$. Because the unrestricted consistency class permits every
$a>0$, the half-width exceeds every finite bound.

\begin{theorem}[Unbounded width without restrictions beyond the observations]
\label{thm:scarce-data}
Fix finite one-step data, its deterministic noise bound $\sigma$, a
certification task $\mathcal Q$, and a candidate control sequence for which the
smooth residual $r_0$ has the data-free rollout neighborhood above at
time $\tau$. For any
coordinate $k$ with $\|Pe_k\|_2>0$, worst-case soundness over
$\cC(\cD_N,\sigma)$ requires
$\rho_P(\cX_{\tau+1})\geq a\|Pe_k\|_2$ for every $a>0$.
An admissible certifier must therefore abstain or return a tube with an infinite
projected diameter. No uniformly sound certifier is
$\bar\rho$-informative for any finite $\bar\rho$ while remaining
unrestricted beyond the observations.
\end{theorem}

Theorem~\ref{thm:scarce-data} relies on the consistency class placing no bound on
the residual away from the observed locations. Restricting the residual
amplitude and derivatives limits the admissible bump, yet the two systems remain
indistinguishable, and propagating their one-step separation to the terminal time
yields a finite plan-dependent lower bound.
Let $a_\star>0$ be small enough that both perturbed residuals remain in the
bounded smooth class, the local perturbation affects the rollout only at the
selected query, and the first-order separation dominates the accumulated
second-order error. Let $C_{T,\tau,P}\geq0$ bound the projected second-order
error over the remaining propagation steps. Define the terminal separation
$\delta_T:=x_T^{(+)}-x_T^{(-)}$.

\begin{theorem}[Plan-dependent width under bounded smooth residuals]
\label{thm:propagated-width}
For a bounded smooth class with $\|r\|_\infty\leq R_{\max}$ and common
derivative bounds, suppose $r_0\pm a\phi e_k$ belong to the class and satisfy
the single-transition localization and derivative conditions in the supplement. For $0<a\leq a_\star$,
every terminal tube sound over this class obeys
\begin{align}
  \rho_P(\cX_T)&\geq
  a\|P G_{T,\tau}e_k\|_2-C_{T,\tau,P}a^2,\notag\\
  G_{T,\tau}&=A_{T-1}\cdots A_{\tau+1},\qquad
  A_s=D_xF_b(x_s^0,u_s),
  \label{eq:width-lower-bound}
\end{align}
where $k$ maximizes $\|PG_{T,\tau}e_k\|_2$ and this gain is nonzero. In the
affine case, $\delta_T=2aG_{T,\tau}e_k$ exactly and
$C_{T,\tau,P}=0$.
\end{theorem}

The matrix $G_{T,\tau}$ maps the one-step state difference at time $\tau+1$ to its
linearized terminal effect, and because the pair stays indistinguishable to $\cA$,
the terminal separation and $2\rho_P(\cX_T)\geq\|P\delta_T\|_2$ yield
Eq.~\eqref{eq:width-lower-bound}, with the full construction and both constants in
the supplement.

\begin{corollary}[Consequence for informative tubes]
\label{cor:planwise-consequences}
Let $\cA$ be worst-case sound over the bounded smooth class in
Theorem~\ref{thm:propagated-width}. If
\begin{equation}
  \max_{0<a\leq a_\star}
  \left[a\|PG_{T,\tau}e_k\|_2-C_{T,\tau,P}a^2\right]
  >\bar\rho,
  \label{eq:planwise-width-threshold}
\end{equation}
then $\cA$ either abstains or is not $\bar\rho$-informative on $\pi$.
\end{corollary}
If some data-consistent trajectory enters the unsafe set, no
worst-case-sound certifier can certify the plan.

\begin{corollary}[Trilemma]
\label{cor:trilemma}
Fix the data, noise bound, task, and plan of Theorem~\ref{thm:scarce-data}. No
admissible deterministic certifier is at once uniformly sound over its residual
class, $\bar\rho$-informative for some finite $\bar\rho$, and unrestricted beyond
the observed locations. A certifier that stays sound while leaving the model
error unrestricted must therefore abstain or return an unbounded projected width
once the plan reaches the data-free query, whereas one that keeps soundness and
finite informative width must restrict the model error beyond the observations.
Restricting the
residual to a bounded smooth class relaxes the third property without dissolving
the tension, since Theorem~\ref{thm:propagated-width} still forces a positive
projected width on every sound tube.
\end{corollary}

\section{Method}
\label{sec:method}

Our method checks observed residual pairs against the supplied bound $L$, bounds
the residual over each queried set with set-membership envelopes
\citep{milanese2004setmembership,jin2020modelinvalidation,jin2022boundedjacobians},
propagates it with the nominal model, and certifies only when the resulting tube
proves safety.
Let $\cZ_{\mathrm{cert}}\subseteq\cZ$ be the compact convex certification domain
containing all observed state-input locations $\{z_i\}_{i=1}^N$, on which both the
residual bound and the nominal-map enclosure, the zonotope over-approximation of
$f_0$ over a query set, are evaluated.

\begin{assumption}[Componentwise residual Lipschitz bound]
\label{ass:regularity}
For a residual $r:\cZ\to\R^n$, for every $k\in\{1,\ldots,n\}$ and all $z,z'\in\cZ_{\mathrm{cert}}$,
\begin{equation}
  |r_k(z)-r_k(z')|\leq\sum_{\ell=1}^{d_z}L_{k\ell}|z_\ell-z'_\ell|,
  \label{eq:regularity}
\end{equation}
where the matrix $L\in\R_{\geq0}^{n\times d_z}$ is supplied before certification.
\end{assumption}

\begin{proposition}[Pairwise consistency test]
\label{prop:falsification}
If there exist $i,j\in\{1,\ldots,N\}$ and $k\in\{1,\ldots,n\}$ such that
\begin{equation}
  |\widetilde r_{i,k}-\widetilde r_{j,k}|>
  \sum_{\ell=1}^{d_z} L_{k\ell}|z_{i,\ell}-z_{j,\ell}|+2\sigma_k,
  \label{eq:falsification}
\end{equation}
then Assumption~\ref{ass:regularity} and the observation-noise bounds cannot
both hold. If no observed pair satisfies Eq.~\eqref{eq:falsification}, the
assumption remains unresolved away from the observed state-input locations.
\end{proposition}

The matrix $L$ and noise bound $\sigma$ are inputs, and
Eq.~\eqref{eq:falsification} checks whether they are jointly consistent with the
observed residuals, with one violating pair causing immediate abstention.

\paragraph{Sources of the Lipschitz bound.}
Our method treats $L$ as side information. On $\cZ_{\mathrm{cert}}$, a sufficient
componentwise choice for a differentiable residual is $L_{k\ell}\geq
\sup_{z\in\cZ_{\mathrm{cert}}}|\partial r_k(z)/\partial z_\ell|$, and certified
bounds of this form come from analytic dynamics with parameter intervals, interval
global optimization \citep{nugroho2022nonlinear}, or certified bounds for a supplied
residual model \citep{fazlyab2019efficient} combined with an independently
bounded model discrepancy. Finite-difference scans and secant slopes between
sampled points instead produce candidate declarations that
Proposition~\ref{prop:falsification} can reject.

Propagation evaluates each residual coordinate over a query set. For observation
$i$, residual coordinate $k$, and $Z\subseteq\cZ_{\mathrm{cert}}$, define the worst-case
Lipschitz distance from $Z$ to sample $z_i$ as
\begin{equation}
  D_{ik}(Z):=\sup_{z\in Z}\sum_{\ell=1}^{d_z}L_{k\ell}|z_\ell-z_{i,\ell}|.
  \label{eq:set-sample-distance}
\end{equation}
Intersecting the observationwise bounds over all observations gives the set
envelope
\begin{align}
  \underline e_k(Z)
  &=\max_i[\widetilde r_{i,k}-\sigma_k-D_{ik}(Z)],\notag\\
  \overline e_k(Z)
  &=\min_i[\widetilde r_{i,k}+\sigma_k+D_{ik}(Z)],\notag\\
  \cE_k(Z)&=[\underline e_k(Z),\overline e_k(Z)].
  \label{eq:set-envelope}
\end{align}
Equation~\eqref{eq:set-envelope} applies the Lipschitz-extension bound of
\citet{mcshane1934extension} and the bounded-noise set-membership envelope of
\citet{milanese2004setmembership} to each residual coordinate. Stacking the
coordinatewise bounds as $\underline e(Z)$ and $\overline e(Z)$ and forming the
residual box $\cE(Z):=\prod_{k=1}^n\cE_k(Z)$, each interval contains every
admissible value of that coordinate over $Z$, so $\cE(Z)$ contains every residual
vector satisfying Assumption~\ref{ass:regularity} and the observation-noise bounds.

Let $\cX_t=\langle c_{X,t},G_{X,t}\rangle$, with the initial set
$\cX_0=\langle c_{X,0},G_{X,0}\rangle$ an axis-aligned box. The fixed-plan query stacks the state
tube with the singleton control through the Cartesian product,
\begin{equation}
  Z_t=\cX_t\times\{u_t\}
  =\left\langle
    \begin{bmatrix}c_{X,t}\\u_t\end{bmatrix},
    \begin{bmatrix}G_{X,t}\\0\end{bmatrix}
  \right\rangle
  =:\langle c_{Z,t},G_{Z,t}\rangle .
  \label{eq:query-zonotope}
\end{equation}
Define the residual midpoint $\mu_t=(\overline e(Z_t)+\underline e(Z_t))/2$ and
half-width $h_t=(\overline e(Z_t)-\underline e(Z_t))/2$. Write
$\mathrm{Box}(Z_t)$ for the interval hull of $Z_t$, and let $J_t=D_zf_0(c_{Z,t})$. Using
interval bounds on the second derivatives of $f_0$ over $\mathrm{Box}(Z_t)$, we
obtain the zero-centered componentwise half-width $q_t\in\R_{\geq0}^n$ that
bounds the nonlinear Taylor remainder, whose closed form and enclosure guarantee are
stated in the supplement. Before order reduction, our method computes
\begin{equation}
  \widehat{\cX}_{t+1}
  =\left\langle
  f_0(c_{Z,t})+\mu_t,
  \left[J_tG_{Z,t}\quad\operatorname{diag}(q_t)\quad
  \operatorname{diag}(h_t)\right]
  \right\rangle.
  \label{eq:propagation}
\end{equation}
After each step, to control the generator count, we apply a standard
order-reduction operator,
\begin{equation}
  \widehat{\cX}_{t+1}\subseteq\cX_{t+1}:=\operatorname{red}(\widehat{\cX}_{t+1}).
  \label{eq:order-reduction}
\end{equation}
The certifier distinguishes pairwise, domain, and safety abstentions, written
$\abstainF$, $\abstainD$, and $\abstainS$, where $\abstainD$ covers either an
interval hull $\mathrm{Box}(Z_t)$ that leaves $\cZ_{\mathrm{cert}}$ or an
interval-Hessian construction that returns no finite nominal-map enclosure.

\begin{algorithm}[t]
\caption{ForeReach certificate}
\label{alg:forereach}
\begin{algorithmic}[1]
  \Require Task $\mathcal Q=(\cX_0,P,\mathcal O,T)$, nominal model $f_0$,
  data $\cD_N$, noise bound $\sigma$, plan $\pi$, and Lipschitz bound $L$
  \Ensure A certificate with tube $\mathbf X$, or a cause-labeled abstention
  \State Compute $\widetilde r_i=x_i^+-f_0(z_i)$
  \State \textbf{if} some $(i,j,k)$ satisfies Eq.~\eqref{eq:falsification} \textbf{then return} $\abstainF$
  \State $\mathbf X\gets(\cX_0)$
  \For{$t=0,\ldots,T-1$}
    \State $Z_t\gets\cX_t\times\{u_t\}$
    \If{$\mathrm{Box}(Z_t)\nsubseteq\cZ_{\mathrm{cert}}$}
      \State \Return $(\abstainD,\tau_D=t)$
    \EndIf
    \State Compute the remainder half-width $q_t$ over $\mathrm{Box}(Z_t)$
    \If{$q_t$ has a non-finite component}
      \State \Return $(\abstainD,\tau_D=t)$
    \EndIf
    \State Construct $\cE(Z_t)$ using Eq.~\eqref{eq:set-envelope}
    \State Form $\widehat{\cX}_{t+1}$ using Eq.~\eqref{eq:propagation}
    \State $\cX_{t+1}\gets\operatorname{red}(\widehat{\cX}_{t+1})$ \Comment{order reduction, Eq.~\eqref{eq:order-reduction}}
    \State Append $\cX_{t+1}$ to $\mathbf X$
  \EndFor
  \If{$R_t\geq m_t$ for some $t\in\{0,\ldots,T\}$}
    \State \Return $(\abstainS,\tau_S=\min\{t:R_t\geq m_t\})$
  \EndIf
  \State \Return certificate with tube $\mathbf X$
\end{algorithmic}
\end{algorithm}
Each plan receives exactly one outcome, and only the certificate asserts
safety.

\begin{theorem}[ForeReach containment]
\label{thm:forereach}
Assume all observed state-input locations belong to $\cZ_{\mathrm{cert}}$, and let
$r\in\cC(\cD_N,\sigma)$ satisfy Assumption~\ref{ass:regularity} on
$\cZ_{\mathrm{cert}}$. For any initial state $x_0\in\cX_0$, suppose the fixed
control sequence is applied and the nominal-map enclosure holds
at every completed step, that is
$\langle f_0(c_{Z,t}),[J_tG_{Z,t}\quad\operatorname{diag}(q_t)]\rangle$ contains
$f_0(Z_t)$. If propagation reaches the horizon $T$, the constructed complete tube
contains the trajectory generated by $f_0+r$. Since $r\in\cC(\cD_N,\sigma)$ and
$x_0\in\cX_0$ were arbitrary, the tube is uniformly sound over this residual
class.
\end{theorem}
The proof is an induction over $t$. The set-membership envelope of
Eq.~\eqref{eq:set-envelope} contains every admissible residual value over $Z_t$,
and the nominal-map enclosure contains $f_0(Z_t)$, so their
zonotope sum in Eq.~\eqref{eq:propagation} contains the next state, order reduction
preserves this containment, and the whole tube is valid once propagation reaches
$T$. Certification then requires $R_t<m_t$ at every safety-relevant time, where
nearby compatible observations shrink the envelope and more obstacle clearance
permits a wider sound tube.

\section{Experiments}
\label{sec:experiments}

Experiments use an affine point-mass and a six-state dynamic-bicycle system,
certifying candidate plans against an unsafe set. Each system is evaluated on two
plan families that separate how a method behaves where the data does and does not
constrain the residual, unsupported plans entering a region with no nearby
observations and corridor-supported plans following a covered corridor. Candidate
plans are hand-designed sequences or fixed action chunks produced by a learned
policy.

We compare against representative constructions of residual uncertainty beyond the
data. Global calibration fits one residual bound to all observations and
applies it uniformly, regional calibration fits one bound per region, and the
comparison also includes a plug-in Gaussian scale, Gaussian-process predictive
bands, the nonlinear Lipschitz reachability algorithm of \citet{alanwar2023noisydata}, and the
conformalized system-level-synthesis MPC budget of \citet{srinivasan2026cpslsmpc}.
The baselines differ in what they assume beyond the observations, the axis the
trilemma isolates.

A complete tube is one propagated through the full horizon $T$. Complete-tube
coverage is the fraction of complete horizon tubes that contain the full target
trajectory. Certified recall is the fraction of target-safe plans that receive a
certificate. False-safe rate is the fraction of certified plans whose target
trajectory is unsafe. Dynamic-bicycle results are ten-seed means, with detailed
setups, certified Lipschitz bounds, certification domains, and secondary metrics in
the supplement.

\subsection{Point-Mass System}

Figure~\ref{fig:exp1-geometry}(a, b) shows the point-mass reachable-tube geometry
for the two plan families.

\begin{figure}[t]
  \centering
  \includegraphics[width=0.9\columnwidth]{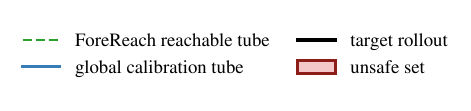}
  \\
  \begin{minipage}[t]{0.49\columnwidth}\centering\includegraphics[width=\linewidth]{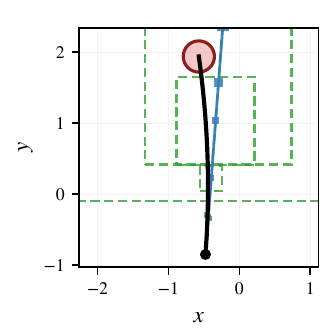}\\[-3pt]{\footnotesize (a)}\end{minipage}\hfill
  \begin{minipage}[t]{0.49\columnwidth}\centering\includegraphics[width=\linewidth]{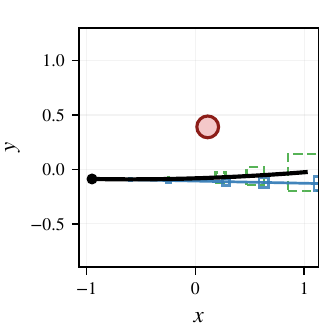}\\[-3pt]{\footnotesize (b)}\end{minipage}
  \\[1pt]
  \begin{minipage}[t]{0.49\columnwidth}\centering\includegraphics[width=\linewidth]{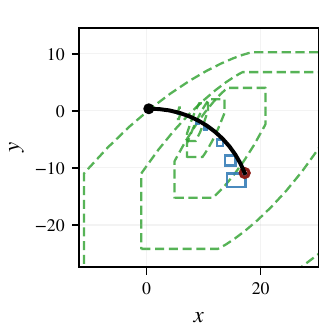}\\[-3pt]{\footnotesize (c)}\end{minipage}\hfill
  \begin{minipage}[t]{0.49\columnwidth}\centering\includegraphics[width=\linewidth]{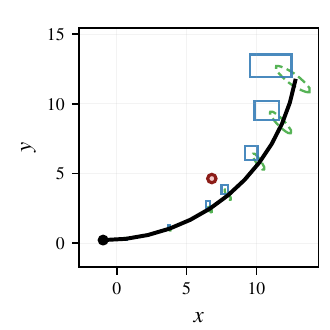}\\[-3pt]{\footnotesize (d)}\end{minipage}
  \caption{Certification geometry in position space. Panels (a) and (b) show the
  point-mass system, and panels (c) and (d) show the dynamic-bicycle system, with
  positions in meters. In the unsupported panels (a) and (c), the global-calibration
  tube remains narrow while the target rollout reaches the
  unsafe set, and our method returns no certificate. In the corridor panels (b)
  and (d), relevant residual observations tighten the green reachable
  tube of our method and enable certification. The green slices are linear position projections
  of the propagated zonotopes of our method.}
  \label{fig:exp1-geometry}
\end{figure}

In the point-mass system, global calibration certifies all unsupported candidates
yet contains none of their target trajectories, and half of its certificates are
false-safe. Every completed tube of our method contains its unsupported trajectory,
so its complete-tube coverage on the unsupported plans is 1.000 with no certificate
issued, and on the supported corridor its complete-tube coverage and certified
recall are 1.000 and 0.998.

\subsection{Dynamic-Bicycle System}

The dynamic-bicycle experiment repeats this comparison in a six-state nonlinear
system, with panels (c, d) of Figure~\ref{fig:exp1-geometry} showing the tubes.

Across the tested noise channels, our method's minimum complete-tube coverage on
the unsupported plans is 1.000, conditional on completing the horizon. At $N=400$,
94.7\% of the unsupported plans exit the certification domain before completing a
tube, and every complete tube for the remaining plans contains its target
trajectory. Its noiseless corridor-supported certified recall rises from
0.490 at $N=25$ to 0.983 at $N=400$. Global calibration covers
0.422 and
0.045 of the unsupported trajectories at $N=25$ and
$N=400$, respectively. The added observations remain in the sampled corridor,
so the calibrated tube narrows without gaining support near the unsupported plans.
Figure~\ref{fig:cross-domain} places the point-mass and
dynamic-bicycle results on the same coverage--recall axes, where the calibration,
plug-in Gaussian, Gaussian-process, and Lipschitz-reachability baselines each fall on the
unsound or the uninformative side while only our method reaches the
sound-and-informative corner. Table~\ref{tab:main-results} reports the
corresponding outcomes of our method.

In our fixed-plan evaluation the remaining baseline, the conformalized
system-level-synthesis MPC (CP-SLS-MPC) budget, yields no usable certificate at any
calibration size, because its conformal radius is infinite for small calibration
sets while the finite large-$N$ bounds remain vacuous
\citep{srinivasan2026cpslsmpc}.

\begin{figure}[t]
  \centering
  \includegraphics[width=0.85\columnwidth]{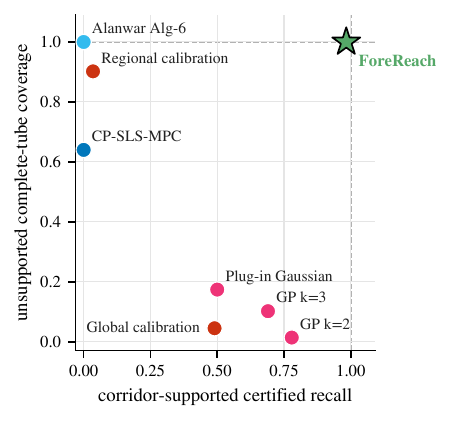}
  \caption{Each point pairs complete-tube coverage on unsupported plans with
  certified recall on corridor-supported plans, where the axes use different
  candidate families and coverage is conditioned on completing the horizon.
  GP $k$ denotes a Gaussian-process band at $k$ predictive standard deviations,
  and Alanwar Alg-6 is the nonlinear Lipschitz method of
  \citet{alanwar2023noisydata}. Our method's abstention and certification outcomes
  appear in Table~\ref{tab:main-results}. The star marks our method.}
  \label{fig:cross-domain}
\end{figure}

\begin{table}[t]
  \centering
  \small
  \caption{ForeReach outcomes. Cov. is complete-tube coverage on the unsupported
  plans, conditioned on completing the horizon. The ordered pair
  $(\abstainD,\abstainS)$ reports domain-exit and safety-abstention rates over
  all unsupported plans. Recall is certified recall on the corridor-supported
  plans, the fraction of target-safe plans that receive a certificate. All
  results are noiseless, and dynamic-bicycle values are ten-seed means.}
  \label{tab:main-results}
  \begin{tabular}{lccc}
\toprule
Domain & Cov. & $(\abstainD,\abstainS)$ & Recall \\
\midrule
Point mass, $N=100$
& $1.000$
& $(0.795,0.205)$
& $0.998$ \\
Dynamic bicycle, $N=400$
& $1.000$
& $(0.947,0.053)$
& $0.983$ \\
\bottomrule
\end{tabular}

\end{table}

\subsection{Certification of Learned Action Chunks}

We now apply the certifier to plans from a learned policy instead of by hand, with
the nominal model, target dynamics, declaration, and certifier unchanged, to test
whether abstention still tracks a plan's distance to the data. This distance is the
maximum over a plan's nominal rollout states of the coordinate-scaled distance to the
nearest observation. Across 45,000 noiseless
evaluations it has Spearman correlation 0.697 with the abstention decision, and the
abstention rate rises from 0.051 in the nearest distance quartile to 0.959 in the
farthest. At $N=400$, placing the same number of observations along the
regions visited by the policy-generated plans raises certified recall from 0.695 to
0.971.

A numerical instantiation of the two data-indistinguishable systems of
Theorem~\ref{thm:propagated-width} reproduces the predicted terminal separation and
yields a positive half-width lower bound on all 64 admissible pairs, with a maximum
of 0.039228, where the supplement gives the full construction and constants.

\subsection{Incorrect Lipschitz Declarations}

To probe a wrong declaration, we supply a Lipschitz bound that underestimates the
residual's regularity. The pairwise consistency test only rejects contradictions
the observations expose, so passing it does not validate the declaration elsewhere.
Here the empirical Lipschitz constant over the tested region is 31.710 times the
declaration, yet the observed samples stay consistent with the smaller bound, so the
test passes.

In the noiseless channel the underestimate makes the propagated tube too narrow to
contain the target trajectory, which reaches the obstacle, so our method issues a
false certificate. Under bounded-uniform noise the same declaration passes again,
but the noise widens the tube until the safety check fires and our method returns
$\abstainS$ at step 7, where the noise forces the abstention with the test still silent. The
test rejects the declaration only once it is halved to a still larger underestimate,
and then on only 9 of the 140 conditions, so a sound certificate depends on a
correct declaration that the one-sided test does not establish.

\section{Related Work}
\label{sec:related}

\paragraph{Trajectory and distributional information.}
Conformalized reachability and conformal planning calibrate trajectory-level
objects and extend that calibration to covariate shift, relaxed
exchangeability, local structure, and sequential distribution change
\citep{huang2026diffusion,huang2026cddr,sun2023plancp,lindemann2023safeplanning,
tibshirani2019covariateshift,barber2023beyondexchangeability,
guan2023localized,bhattacharyya2026groupweighted,
fannjiang2022feedbackcovariateshift,gibbs2021adaptiveconformal}.
Control-oriented variants carry the same calibration into neural MPC, learned
Koopman reachable sets, and system-level synthesis
\citep{wang2025conformaltubempc,nath2026koopmanconformal,
zhang2026transformeroutput,srinivasan2026cpslsmpc}.
Each draws its guarantee from repeated trajectory-level outcomes, whether
calibration data, a deployment assumption, or delayed online labels, none of which
the isolated one-step transitions here supply.

\paragraph{Structural information.}
Set-membership identification and data-driven reachability combine bounded
noise with regularity classes, while reachset-conformant identification learns
uncertainty sets whose reachable outputs contain measured behavior
\citep{alanwar2021matrixzonotopes,alanwar2023noisydata,
zhang2026interpolated,luetzow2026reachsetconformant}.
DaTaReach uses trajectory data and declared regularity for reachability
overapproximation \citep{djeumou2023onthefly}, whereas we certify fixed sequences
from isolated one-step transitions, and deterministic kernel bounds with
robust-control model validation similarly use side information to build
finite-sample envelopes or test model consistency
\citep{maddalena2021deterministic,smith1992modelvalidation}.

Nonlinear set-membership envelopes, pairwise invalidation, bounded-Jacobian tools,
and related MPC bounds construct residual uncertainty from a declared regularity
class
\citep{milanese2004setmembership,jin2020modelinvalidation,
jin2022boundedjacobians,canale2014smpc,manzano2020robust,
manzano2021hoelder}, which our method uses inside a pre-execution certifier that
clears a fixed sequence before deployment. Other safe-learning work estimates
model-error regularity statistically or combines reachability or Gaussian-process
bounds with feedback control, whereas our lower bound concerns the projected width
forced on a fixed sequence
\citep{knuth2021learneddynamics,knuth2022correction,
chou2022contractionmetrics,knuth2023statisticalsafety,
fisac2019safety,koller2018safempc}.

\paragraph{Information limits.}
Cai and Low study adaptation limits for honest confidence intervals, and
Barber et al. establish limits for distribution-free conditional prediction
\citep{cai2004adaptation,barber2021limits}.
Dietrich et al. show that deterministic data-driven reachability can require
exponentially many samples in the state dimension
\citep{dietrich2025holdoutreachability}, and information-based complexity, minimax
estimation, and set-membership theory provide related finite-information
lower-bound frameworks
\citep{traub1988information,tsybakov2009nonparametric,milanese1985optimal}.

\section{Conclusion}
\label{sec:scope}

Certifying a fixed plan from scarce one-step data is limited less by data quantity
than by what the data leave unconstrained. Once the plan leaves the observed region,
no deterministic certifier can at once be uniformly sound, keep the projected tube
finite, and leave the model error unrestricted, and bounded smoothness turns this
tension into a finite plan-dependent width that any sound tube must carry, so a
usable certificate needs structure supplied from outside the data, which the data
can refute but never confirm. The experiments make the same point, where samples accumulating away from the plan
lose containment while observations placed where the plan travels recover
certification, including for learned action chunks. A useful certificate is decided
by where the target data constrain the plan and how uncertainty propagates, with the
sample count secondary, and the guarantee is conditional on a valid supplied bound.
\label{tech:end}

\bibliography{references}

\clearpage
\appendix
\setcounter{figure}{0}
\setcounter{table}{0}
\setcounter{equation}{0}
\setcounter{theorem}{0}
\setcounter{proposition}{0}
\setcounter{definition}{0}
\setcounter{corollary}{0}
\setcounter{remark}{0}
\renewcommand{\thefigure}{S\arabic{figure}}
\renewcommand{\thetable}{S\arabic{table}}
\renewcommand{\theequation}{S\arabic{equation}}
\renewcommand{\thetheorem}{S\arabic{theorem}}
\renewcommand{\theproposition}{S\arabic{proposition}}
\renewcommand{\thecorollary}{S\arabic{corollary}}
\renewcommand{\thedefinition}{S\arabic{definition}}
\renewcommand{\theremark}{S\arabic{remark}}

\section{Information Model}

\begin{definition}[Nominal-model residual system and scarce data]
Let $\mathsf{X}=\R^n$, $\mathsf{U}\subseteq\R^m$, and
$\cZ=\mathsf{X}\times\mathsf{U}$. A nominal dynamics model
$f_0:\cZ\to\mathsf{X}$ is available before target-system certification.
The unknown target dynamics have the decomposition
\begin{equation}
  x^+ = f_\star(x,u)
  = f_0(x,u)+r_\star(x,u),
  \qquad z=(x,u),
  \label{eq:nominal-residual}
\end{equation}
where $r_\star:\cZ\to\R^n$ is the residual function.

The available target-system data are $N$ isolated one-step transitions,
with residual observations computed relative to $f_0$,
\begin{equation}
  \cD_N
  = \left\{(z_i,x_i^+)\right\}_{i=1}^{N},
  \qquad
  \widetilde r_i
  = x_i^+-f_0(z_i)
  = r_\star(z_i)+\eta_i .
  \label{eq:residual-data}
\end{equation}
For every residual coordinate $k\in\{1,\ldots,n\}$, the observation noise
satisfies the deterministic componentwise bound
\begin{equation}
  |\eta_{i,k}|\leq \sigma_k .
  \label{eq:noise-bound}
\end{equation}
The sample locations $z_i$ are finite and pairwise distinct for the smooth
two-system construction.
\end{definition}

\begin{definition}[Candidate control sequence, certification task, and induced queries]
The candidate plan is the open-loop control sequence
\begin{equation}
  \pi=(u_0,\ldots,u_{T-1}),
  \label{eq:plan}
\end{equation}
and the certification task is
$\mathcal Q=(\cX_0,P,\mathcal O,T)$, where
$P\in\R^{p\times n}$ is a nonzero safety-relevant state projection and
$\mathcal O\subseteq\R^p$. The two adversarial systems start from the same
$x_0\in\cX_0$.

For any residual $r$, let
\begin{equation}
  x_{t+1}^{r}=f_0(x_t^{r},u_t)+r(x_t^{r},u_t),
  \qquad
  z_t^{r}=(x_t^{r},u_t),
  \label{eq:plan-rollout}
\end{equation}
and define the induced query sequence
$Q_\pi(r)=\{z_t^{r}\}_{t=0}^{T-1}$.
\end{definition}

\paragraph{Information-model extensions.}
The certifier input tuple is $(f_0,\cD_N,\sigma,\pi)$, with the certification
task fixed separately. Richer information models add complete calibration
trajectories, deployment-distribution or likelihood-ratio information, online
deployment labels, or a feedback policy that changes the fixed plan.

\begin{definition}[Data-consistent residual class]
The unrestricted consistency class is
\begin{equation}
  \cC(\cD_N,\sigma)
  =
  \left\{
    r:\cZ\to\R^n:
    \begin{array}{l}
    \text{there exist }\eta_1,\ldots,\eta_N\text{ with}\\
    \widetilde r_i=r(z_i)+\eta_i\text{ and}\\
    |\eta_{i,k}|\leq\sigma_k\text{ for every }i,k
    \end{array}
  \right\}.
  \label{eq:consistency-class}
\end{equation}
No relation is imposed between values of $r$ away from the observed
state-input locations.
\end{definition}

\section{Admissible Certifiers and Three Properties}

\begin{definition}[Admissible pre-execution certifier]
For a fixed certification task, an admissible deterministic certifier receives
$(f_0,\cD_N,\sigma,\pi)$ and returns either abstention or a tube,
\begin{equation}
  \cA(f_0,\cD_N,\sigma,\pi)
  \in
  \left\{\abstain\right\}
  \cup
  \left(2^{\mathsf{X}}\right)^{T+1}.
  \label{eq:admissible-program}
\end{equation}
When the output is a tube, write it as
$\mathbf{X}=(\cX_0,\ldots,\cX_T)$. Obstacle avoidance can be checked after
propagation by testing each safety projection against the unsafe set.
\end{definition}

\begin{definition}[Worst-case soundness]
Fix $(f_0,\cD_N,\sigma,\pi)$ and a certification task. A non-abstaining output
$\mathbf{X}$ is uniformly sound over the unrestricted consistency class if
\begin{equation}
  \forall r\in\cC(\cD_N,\sigma),\quad
  \forall t\in\{0,\ldots,T\},\quad
  x_t^{r}\in\cX_t.
  \label{eq:worst-case-soundness}
\end{equation}
The universal quantifier ranges over every true system consistent with the
fixed observed data and deterministic noise bound.
\end{definition}

\begin{definition}[Task-relevant informativeness]
For a set $S\subseteq\mathsf{X}$, define its projected diameter and half-width by
\begin{equation}
  \begin{gathered}
  \diam_P(S):=\sup_{x,y\in S}\|P(x-y)\|_2,\\
  \rho_P(S):=\tfrac{1}{2}\diam_P(S).
  \end{gathered}
  \label{eq:projected-width}
\end{equation}
A certifier is $\bar\rho$-informative for the task if it returns a tube and
\begin{equation}
  \max_{0\leq t\leq T}\rho_P(\cX_t)\leq\bar\rho<\infty.
  \label{eq:informativeness}
\end{equation}
For a centered projected slice, write $P\cX_t=c_t\oplus E_t$ with
$E_t=-E_t$, and define
\begin{equation}
  R_t:=\sup_{e\in E_t}\|e\|_2,
  \qquad
  m_t:=\operatorname{dist}(c_t,\mathcal O).
  \label{eq:safety-radius-margin-supp}
\end{equation}
The condition $R_t<m_t$ is sufficient for safety. For the centrally symmetric
projected slices used here, including zonotopes, $R_t=\rho_P(\cX_t)$ because
$\diam_P(\cX_t)=2\sup_{e\in E_t}\|e\|_2$.
For a fixed plan, define the plan-level worst-case certification indicator
\begin{equation}
  \operatorname{Cert}^{\mathrm{wc}}_{\pi}(\cA)
  :=\mathbf 1\!\left\{\cA\text{ returns a safety certificate on }\pi\right\}.
  \label{eq:planwise-certification}
\end{equation}
This fixed-plan indicator is distinct from experimental certified recall,
which averages certification decisions over candidate plans labeled safe by
target-system Monte Carlo rollouts.
\end{definition}

\begin{definition}[Unrestricted beyond the observations]
An admissible certifier is unrestricted beyond the observations in Theorem~\ref{thm:trilemma} when its
soundness claim uses the full consistency class
$\cC(\cD_N,\sigma)$. Lipschitz, bounded-variation, RKHS,
parametric, shift, likelihood-ratio, trajectory-law, and global-amplitude
restrictions define strict subclasses outside this property.
\end{definition}

\begin{remark}[Why pairwise tests contain $2\sigma_k$ and envelopes contain $\sigma_k$]
Suppose the true residual obeys the componentwise declaration
\begin{equation}
  |r_k(z)-r_k(z')|
  \leq \sum_{\ell=1}^{n+m}L_{k\ell}|z_\ell-z_\ell'|.
  \label{eq:declared-regularity}
\end{equation}
Comparing two noisy observations gives
\begin{align}
  |\widetilde r_{i,k}-\widetilde r_{j,k}|
  &\leq |r_k(z_i)-r_k(z_j)|
    +|\eta_{i,k}|+|\eta_{j,k}| \notag\\
  &\leq
  \sum_{\ell}L_{k\ell}|z_{i,\ell}-z_{j,\ell}|
  +2\sigma_k .
  \label{eq:two-noise-slacks}
\end{align}
Transferring one observation to one query gives
\begin{align}
  |r_k(z)-\widetilde r_{i,k}|
  &\leq |r_k(z)-r_k(z_i)|+|\eta_{i,k}| \notag\\
  &\leq
  \sum_{\ell}L_{k\ell}|z_\ell-z_{i,\ell}|+\sigma_k .
  \label{eq:one-noise-slack}
\end{align}
The two adversarial systems used in Theorem~\ref{thm:trilemma} share the same
admissible noise realization $\eta_i=\widetilde r_i-r_0(z_i)$, where $r_0$ is the
baseline residual defined below, which satisfies
$|\eta_{i,k}|\le\sigma_k$ because $r_0\in\cC(\cD_N,\sigma)$. Their observed
datasets are therefore identical without an additional $\sigma_k$ term.
\end{remark}

\section{Data-Free Neighborhood Along the Baseline Rollout}

\begin{definition}[Data-free rollout neighborhood and bump]
Let $r_0$ be the restriction to $\cZ$ of a $C^\infty$ map defined on an open
neighborhood of $\cZ$ and assume $r_0\in\cC(\cD_N,\sigma)$. Set
$F_b=f_0+r_0$, and let $x_t^0$ and $z_t^0=(x_t^0,u_t)$ denote the trajectory
and queries of $F_b$ under $\pi$.

The baseline rollout has a data-free neighborhood at
$\tau\in\{0,\ldots,T-1\}$ if there is a $d_{\mathrm{gap}}>0$ such that
\begin{equation}
  \min_{1\leq i\leq N}\|z_\tau^0-z_i\|_2>d_{\mathrm{gap}},
  \label{eq:sample-gap}
\end{equation}
and $z_\tau^0$ is isolated from the other finite baseline queries,
\begin{equation}
  d_{\mathrm{query}}
  :=
  \min_{\substack{0\leq s<T\\s\neq\tau}}
  \|z_\tau^0-z_s^0\|_2
  >0,
  \label{eq:query-isolation}
\end{equation}
with $d_{\mathrm{query}}=\infty$ when the minimum is over an empty set.
Choose
\begin{equation}
  0<\varepsilon<
  \min\{d_{\mathrm{gap}},d_{\mathrm{query}}/2\}
  \label{eq:bump-radius}
\end{equation}
and a smooth compactly supported bump
$\phi\in C_c^\infty(\R^{n+m})$ satisfying
\begin{equation}
  0\leq\phi\leq 1,\qquad
  \phi(z_\tau^0)=1,\qquad
  \supp(\phi)\subset B_\varepsilon(z_\tau^0).
  \label{eq:bump-properties}
\end{equation}
For a state coordinate vector $e_k$ and amplitude $a>0$, define
\begin{equation}
  r^{(+)}(z)=r_0(z)+a\phi(z)e_k,
  \qquad
  r^{(-)}(z)=r_0(z)-a\phi(z)e_k.
  \label{eq:adversarial-residuals}
\end{equation}
\end{definition}

\paragraph{Operator norms of derivative maps.}
For a $j$-linear map $A:(\R^{n+m})^j\to\R^n$ we write
\[
  \|A\|_{2\to2}
  :=\sup\{\|A[v_1,\ldots,v_j]\|_2:\ \|v_i\|_2\leq1\ \forall i\},
\]
so that $\|D_z r(z)\|_{2\to2}$ is the induced operator norm of the Jacobian and
$\|D_z^2 r(z)\|_{2\to2}=\sup_{\|v\|_2\leq1,\,\|w\|_2\leq1}\|D_z^2 r(z)[v,w]\|_2$ is
the norm of the Hessian bilinear map. Every first- and second-derivative bound
below uses this norm.

\begin{definition}[Bounded smooth consistency class]
For a declared amplitude bound $R_{\max}>0$ and derivative bounds
$\mathbf{H}=(H_1,H_2)$, define
\begin{equation}
  \begin{aligned}
  &\cC_{\mathrm{sm}}(\cD_N,\sigma;R_{\max},\mathbf H)
  :=\Big\{\,r\in\cC(\cD_N,\sigma):\\
  &\quad
  \begin{array}{l}
  r\text{ extends to a }C^2\text{ map near }\cZ,\\
  \sup_{z\in\cZ}\|r(z)\|_\infty\leq R_{\max},\\
  \sup_{z\in\cZ}\|D_z r(z)\|_{2\to2}\leq H_1,\\
  \sup_{z\in\cZ}\|D_z^2 r(z)\|_{2\to2}\leq H_2
  \end{array}\Big\}.
  \end{aligned}
  \label{eq:bounded-smooth-class}
\end{equation}
For the two-system construction, class membership is an explicit premise,
\begin{equation}
  0<a\leq R_{\max},
  \qquad
  r_0+a\phi e_k\in\cC_{\mathrm{sm}},
  \qquad
  r_0-a\phi e_k\in\cC_{\mathrm{sm}}.
  \label{eq:bump-class-membership}
\end{equation}
Thus $a$ must fit both the amplitude slack around $r_0$ and the declared
first- and second-derivative bounds of the class.
\end{definition}

\begin{definition}[Single-transition localization condition]
For the multi-step quantitative statement, the bump and amplitude satisfy the
single-transition localization condition if the two perturbed trajectories coincide with the baseline
trajectory through time $\tau$, both evaluate $\phi(z_\tau^0)=1$ on the
transition $\tau\to\tau+1$, and
\begin{equation}
  z_s^{(\pm)}\notin\supp(\phi),
  \qquad
  s=\tau+1,\ldots,T-1.
  \label{eq:one-pass}
\end{equation}
Thus the two systems evolve under the common baseline map $F_b$ after the
perturbed transition.
\end{definition}

\section{Main Statement}

Theorem 1 and Theorem 2 of the main text correspond to Level A and Level B of
the combined statement below.

\begin{theorem}[Finite-data reachable-tube width lower bounds]
\label{thm:trilemma}
Fix a finite dataset $\cD_N$, its deterministic noise bound $\sigma$, an
open-loop control sequence $\pi$, and a certification task $\mathcal Q$.
Assume that a smooth feasible residual $r_0$ has a data-free rollout
neighborhood at some $\tau$. Let $\cA$ be any admissible
pre-execution certifier.

\smallskip
\noindent\textbf{Level A: unrestricted residual class.}
If $\cA$ is worst-case sound over the full consistency class
$\cC(\cD_N,\sigma)$, then $\cA$ is not $\bar\rho$-informative for any finite
$\bar\rho$. More precisely, either $\cA$ abstains or at least one
safety-projected tube slice has infinite diameter.

For every coordinate $k$ with $\|Pe_k\|_2>0$ and every $a>0$, the two residuals
in \eqref{eq:adversarial-residuals} belong to
$\cC(\cD_N,\sigma)$. Their trajectories agree through time $\tau$ and satisfy
the exact next-state separation
\begin{equation}
  \delta_{\tau+1}
  :=
  x_{\tau+1}^{(+)}-x_{\tau+1}^{(-)}
  =2a e_k .
  \label{eq:first-separation}
\end{equation}
Every tube sound for the full consistency class satisfies
\begin{equation}
  \rho_P(\cX_{\tau+1})
  \geq
  \frac{1}{2}\|P\delta_{\tau+1}\|_2
  =
  a\|Pe_k\|_2 .
  \label{eq:unbounded-halfwidth}
\end{equation}
The bound holds for arbitrary $a$, so no finite $\bar\rho$ can satisfy
worst-case soundness. No admissible certifier can therefore be simultaneously
worst-case sound, $\bar\rho$-informative, and unrestricted beyond the observations.

\smallskip
\noindent\textbf{Level B: bounded smooth consistency class.}
For
\begin{equation}
  A_s:=D_xF_b(x_s^0,u_s),
  \qquad
  G_{T,\tau}
  :=
  A_{T-1}A_{T-2}\cdots A_{\tau+1},
  \label{eq:propagation-gain}
\end{equation}
use the convention $G_{\tau+1,\tau}=I$. The product is time ordered from the
perturbed state at time $\tau+1$ to the terminal state at time $T$. Choose
\begin{equation}
  k\in
  \operatorname*{arg\,max}_{1\leq j\leq n}
  \|PG_{T,\tau}e_j\|_2
  \quad\text{and require}\quad
  \|PG_{T,\tau}e_k\|_2>0.
  \label{eq:nondegenerate-direction}
\end{equation}
If $F_b(\cdot,u_s)$ is affine for every
$s=\tau+1,\ldots,T-1$, then
for every amplitude satisfying \eqref{eq:bump-class-membership} and the
single-transition localization condition,
\begin{equation}
  \delta_T
  =
  2aG_{T,\tau}e_k
  \label{eq:affine-terminal-separation}
\end{equation}
and every tube sound over
$\cC_{\mathrm{sm}}(\cD_N,\sigma;R_{\max},\mathbf H)$ obeys the exact lower bound
\begin{equation}
  \rho_P(\cX_T)
  \geq
  a\|PG_{T,\tau}e_k\|_2.
  \label{eq:affine-width-lower-bound}
\end{equation}

For the smooth nonlinear case, let
$\Omega_s=B_{\varrho_s}(x_s^0)$, $\varrho_s>0$, be regions on which the known nominal
derivative bounds
\begin{equation}
  \begin{gathered}
  \ell_{0,s}:=\sup_{x\in\Omega_s}\|D_xf_0(x,u_s)\|_{2\to2},\\
  h_{0,s}:=\sup_{x\in\Omega_s}\|D_x^2f_0(x,u_s)\|_{2\to2}
  \end{gathered}
  \label{eq:nominal-derivative-budgets}
\end{equation}
are finite. Define
\begin{equation}
  \begin{gathered}
  \Lambda_s:=\ell_{0,s}+H_1,\qquad
  \overline H_s:=h_{0,s}+H_2,\\
  M_{\tau+1}:=1,\qquad
  M_s:=\prod_{q=\tau+1}^{s-1}\Lambda_q .
  \end{gathered}
  \label{eq:deviation-gains}
\end{equation}
For $s=\tau+1,\ldots,T-1$, also set
\begin{equation}
  \begin{gathered}
  G_{T,s+1}:=A_{T-1}\cdots A_{s+1},\qquad
  G_{T,T}:=I,\\
  d_s:=\operatorname{dist}(z_s^0,\supp\phi).
  \end{gathered}
  \label{eq:tail-gains-clearance}
\end{equation}

Assume the class slacks
\begin{align}
  \mu_0&:=R_{\max}-\sup_{z\in\cZ}\|r_0(z)\|_\infty>0,\notag\\
  \mu_1&:=H_1-\sup_{z\in\cZ}\|D_zr_0(z)\|_{2\to2}>0,\notag\\
  \mu_2&:=H_2-\sup_{z\in\cZ}\|D_z^2r_0(z)\|_{2\to2}>0.
  \label{eq:class-slacks}
\end{align}
Let $b_0:=\sup_z|\phi(z)|$ and
$b_j:=\sup_z\|D_z^j\phi(z)\|_{2\to2}$ for $j=1,2$, with the convention
$\mu_j/b_j=\infty$ when $b_j=0$. Define
\begin{align}
  a_{\mathrm{class}}
  &:=
  \min_{j=0,1,2}\frac{\mu_j}{b_j},
  \label{eq:aclass}\\
  a_{\mathrm{pass}}
  &:=
  \frac12
  \min_{s=\tau+1,\ldots,T-1}
  \frac{\min\{d_s,\varrho_s\}}{M_s},
  \label{eq:apass}\\
  C_{T,\tau,P}
  &:=
  \frac12
  \sum_{s=\tau+1}^{T-1}
  \|PG_{T,s+1}\|_{2\to2}\,
  \overline H_s M_s^2,
  \label{eq:explicit-C}\\
  a_{\mathrm{lin}}
  &:=
  \begin{cases}
  \|PG_{T,\tau}e_k\|_2/(2C_{T,\tau,P}),
    & C_{T,\tau,P}>0,\\
  \infty, & C_{T,\tau,P}=0,
  \end{cases}
  \label{eq:alin}\\
  a_\star
  &:=
  \min\{a_{\mathrm{class}},a_{\mathrm{pass}},a_{\mathrm{lin}}\}.
  \label{eq:astar}
\end{align}
An empty minimum in \eqref{eq:apass} is $\infty$, and a ratio with $M_s=0$
is also $\infty$. If $F_b$ is $C^2$ on these regions, then every
$0<a\leq a_\star$ satisfies the class-membership,
single-transition localization, and $a\leq a_{\mathrm{lin}}$ conditions, and every tube sound over
$\cC_{\mathrm{sm}}(\cD_N,\sigma;R_{\max},\mathbf H)$ satisfies
\begin{equation}
  \|P\delta_T\|_2
  \geq
  2a\|PG_{T,\tau}e_k\|_2
  -2C_{T,\tau,P}a^2,
  \qquad 0<a\leq a_\star,
  \label{eq:nonlinear-separation}
\end{equation}
and hence
\begin{equation}
  \rho_P(\cX_T)
  \geq
  a\|PG_{T,\tau}e_k\|_2
  -C_{T,\tau,P}a^2.
  \label{eq:nonlinear-width-lower-bound}
\end{equation}

For either quantitative variant, let
$D_T:=\|P\delta_T\|_2$. If a reported terminal half-width is capped by
$\bar\rho$, soundness for both data-consistent systems is impossible whenever
\begin{equation}
  D_T>2\bar\rho .
  \label{eq:correct-width-contradiction}
\end{equation}
\end{theorem}

\begin{remark}[Gap dependence of $a_\star$]
Level A uses the unrestricted consistency class and places no amplitude cap on
$a$. Any positive data-free gap therefore supports the unbounded-amplitude
alternative. For Level B, let a fixed unit-scale bump $\psi$ generate
$\phi_h(z)=\psi((z-z_\tau^0)/h)$ inside a gap of radius $h$. Then
\begin{equation}
  \|D^j\phi_h\|_\infty=h^{-j}\|D^j\psi\|_\infty,
  \qquad j=0,1,2,
\end{equation}
and the class-membership threshold becomes
\begin{equation}
  a_{\mathrm{class}}(h)=
  \min\left\{
  \frac{\mu_0}{\|\psi\|_\infty},
  \frac{\mu_1h}{\|D\psi\|_\infty},
  \frac{\mu_2h^2}{\|D^2\psi\|_\infty}
  \right\}.
\end{equation}
The localization threshold $a_{\mathrm{pass}}(h)$ and the linear threshold
$a_{\mathrm{lin}}$ combine with $a_{\mathrm{class}}(h)$ as at full horizon. Narrower
gaps reduce the admissible perturbation through the bump-derivative bounds and the
post-split clearance, which is the quantitative counterpart of adding support near
the plan query.
\end{remark}

\begin{remark}[Quantifiers and construction]
Theorem~\ref{thm:trilemma} quantifies over every admissible certifier claiming
worst-case soundness, and its lower bound comes from two data-indistinguishable
systems with a multi-step propagation gain, which distinguishes it from a
single-estimator or single-point lower bound. Confidence-set impossibility results
give statistical context \citep{cai2004adaptation,barber2021limits}, data-driven
abstraction and model invalidation give comparisons at the level of function classes
\citep{jin2020modelinvalidation,jin2022boundedjacobians}, and estimated
Lipschitz-trust-region planning is a separate route whose guarantees are probabilistic
\citep{knuth2021learneddynamics,knuth2022correction}.
\end{remark}

\begin{remark}[Randomized procedures]
If a randomized certifier receives bitwise identical inputs under two systems,
then its output distribution is identical under those systems. Theorem
\ref{thm:trilemma} covers deterministic certifiers.
\end{remark}

\begin{remark}[Short-horizon lower-bound scan]
For the fixed six-state system, we scan every admissible pair of bump time
$\tau\leq T-2$ and coordinate $k$ that has a positive data-free gap and a positive
projected gain. All 64 such pairs give a positive proved terminal half-width lower
bound, the largest 0.039228 at $\tau=10$ and $k=2$, since a shorter remaining
propagation lowers the remainder constant and widens the admissible amplitude.
\end{remark}

\section{Proofs of the Main-Text Width Bounds}

\begin{proof}[Proof of Theorem~\ref{thm:trilemma}]
\medskip\noindent\textbf{Sample-point consistency and indistinguishability.}
By \eqref{eq:sample-gap} and \eqref{eq:bump-radius}, every sample point $z_i$ lies
outside $\supp\phi$. Hence
\begin{equation}
  r^{(+)}(z_i)=r_0(z_i)=r^{(-)}(z_i)
  \qquad (i=1,\ldots,N).
  \label{eq:proof-sample-identity}
\end{equation}
Because $r_0\in\cC(\cD_N,\sigma)$, there is an admissible noise realization
$\eta_i=\widetilde r_i-r_0(z_i)$ with $|\eta_{i,k}|\le\sigma_k$, and the same
realization makes both perturbed residuals consistent with the observed dataset.
Therefore the deterministic certifier $\cA$ receives identical inputs and
returns the same tube or the same abstention decision under both systems.

\medskip\noindent\textbf{The trajectories meet the bump.}
At time zero, $x_0^{(+)}=x_0^0=x_0^{(-)}$. Suppose the three states agree at a
time $t<\tau$. Then $z_t^0\notin\supp\phi$ by
\eqref{eq:query-isolation} to \eqref{eq:bump-properties}, so
$r^{(+)}(z_t^0)=r_0(z_t^0)=r^{(-)}(z_t^0)$.
Equation \eqref{eq:plan-rollout} gives equality at time $t+1$.

Induction yields $x_t^{(+)}=x_t^0=x_t^{(-)}$ for $t\leq\tau$.
At $z_\tau^0$, \eqref{eq:bump-properties} gives $\phi(z_\tau^0)=1$, and hence
\begin{equation}
  x_{\tau+1}^{(\pm)}
  =
  x_{\tau+1}^0\pm ae_k,
  \qquad
  \delta_{\tau+1}=2ae_k.
  \label{eq:proof-first-split}
\end{equation}

\medskip\noindent\textbf{Level A.}
Because $P\neq0$, at least one coordinate satisfies $\|Pe_k\|_2>0$.
Soundness over $\cC(\cD_N,\sigma)$ requires the common slice
$\cX_{\tau+1}$ to contain both states in \eqref{eq:proof-first-split}.
The diameter definition \eqref{eq:projected-width} therefore gives
\begin{equation}
  \rho_P(\cX_{\tau+1})
  \geq
  \frac12\|P\delta_{\tau+1}\|_2
  =
  a\|Pe_k\|_2.
  \label{eq:proof-level-A-width}
\end{equation}
The unrestricted class permits every $a>0$. A finite projected tube cannot
satisfy \eqref{eq:proof-level-A-width} for all $a$, so a sound certifier must
abstain or fail $\bar\rho$-informativeness for every finite $\bar\rho$.

\medskip\noindent\textbf{Level B, affine propagation.}
Under the single-transition localization condition, both residuals equal $r_0$ after the perturbed
transition. Thus both trajectories evolve under the same maps
$F_b(\cdot,u_s)$ for $s\geq\tau+1$.
If these maps are affine, subtraction gives
\begin{equation}
  \delta_{s+1}=A_s\delta_s,
  \qquad s=\tau+1,\ldots,T-1.
  \label{eq:proof-affine-recursion}
\end{equation}
Iterating \eqref{eq:proof-affine-recursion} from
\eqref{eq:proof-first-split} yields
$\delta_T=2aG_{T,\tau}e_k$. A tube sound over
$\cC_{\mathrm{sm}}$ contains both terminal states, so its half-width is at
least $a\|PG_{T,\tau}e_k\|_2$.

\medskip\noindent\textbf{Level B, class membership and single-transition localization threshold.}
Let $\|g\|_{(0)}:=\sup_z\|g(z)\|_\infty$ and, for $j=1,2$, let
$\|g\|_{(j)}:=\sup_z\|D_z^jg(z)\|_{2\to2}$. The triangle inequality and
\eqref{eq:aclass} give
\begin{equation}
  \begin{gathered}
  \|r_0\pm a\phi e_k\|_{(j)}
  \leq\|r_0\|_{(j)}+ab_j\\
  \leq\|r_0\|_{(j)}+\mu_j
  \leq K_j,\quad j=0,1,2,
  \end{gathered}
  \label{eq:proof-class-membership}
\end{equation}
where $K_0=R_{\max}$, $K_1=H_1$, and $K_2=H_2$. The bump is zero at every
sample point, so \eqref{eq:proof-class-membership} proves membership in
$\cC_{\mathrm{sm}}$.

Let $h_s^{(\pm)}:=x_s^{(\pm)}-x_s^0$. At time $\tau+1$,
$\|h_{\tau+1}^{(\pm)}\|_2=a=aM_{\tau+1}$.
We prove the deviation bound and the single-transition localization property simultaneously. Suppose
$\|h_s^{(\pm)}\|_2\leq aM_s$ at some
$s\in\{\tau+1,\ldots,T-1\}$. Since $a\leq a_{\mathrm{pass}}$,
$aM_s\leq\min\{d_s,\varrho_s\}/2$. It follows that
\begin{equation}
  \operatorname{dist}(z_s^{(\pm)},\supp\phi)
  \geq d_s-\|h_s^{(\pm)}\|_2
  \geq d_s/2>0,
  \label{eq:proof-one-pass}
\end{equation}
and the segment from $x_s^0$ to $x_s^{(\pm)}$ lies in $\Omega_s$. The bump
term is therefore zero at $z_s^{(\pm)}$, so the perturbed trajectory uses the
common map $F_b(\cdot,u_s)$. The mean-value inequality now gives
\begin{equation}
  \|h_{s+1}^{(\pm)}\|_2
  \leq \Lambda_s\|h_s^{(\pm)}\|_2
  \leq aM_{s+1}.
  \label{eq:proof-rough-deviation}
\end{equation}
Starting from the base case at $\tau+1$, this simultaneous induction establishes
$\|h_s^{(\pm)}\|_2\leq aM_s$ through time $T$ and
\eqref{eq:proof-one-pass} through time $T-1$. Thus neither perturbed trajectory
re-enters the bump support after time $\tau+1$.

\medskip\noindent\textbf{Level B, nonlinear separation.}
Taylor expansion of the common map around $x_s^0$ gives
\begin{equation}
  h_{s+1}^{(\pm)}
  =
  A_sh_s^{(\pm)}+q_s^{(\pm)},
  \qquad
  \|q_s^{(\pm)}\|_2
  \leq
  \frac12\overline H_s\|h_s^{(\pm)}\|_2^2.
  \label{eq:proof-taylor-two-trajectories}
\end{equation}
Here $\overline H_s=h_{0,s}+H_2$ follows from
$F_b=f_0+r_0$ and the second-derivative bound in
\eqref{eq:bounded-smooth-class}.

Subtracting the two expansions in
\eqref{eq:proof-taylor-two-trajectories} yields
\begin{equation}
  \begin{gathered}
  \delta_{s+1}=A_s\delta_s+w_s,\qquad
  w_s:=q_s^{(+)}-q_s^{(-)},\\
  \|w_s\|_2\leq\overline H_s a^2M_s^2.
  \end{gathered}
  \label{eq:proof-separation-recursion}
\end{equation}
The last inequality applies \eqref{eq:proof-rough-deviation} separately to the
two individual distances from the baseline.

Unrolling \eqref{eq:proof-separation-recursion} gives
\begin{equation}
  \delta_T
  =
  2aG_{T,\tau}e_k
  +
  \sum_{s=\tau+1}^{T-1}G_{T,s+1}w_s.
  \label{eq:proof-unrolled-separation}
\end{equation}
Applying $P$, the reverse triangle inequality, and
\eqref{eq:proof-separation-recursion} gives
\begin{align}
  \|P\delta_T\|_2
  &\geq
  2a\|PG_{T,\tau}e_k\|_2\notag\\
  &\quad-
  a^2\sum_{s=\tau+1}^{T-1}
  \|PG_{T,s+1}\|_{2\to2}\overline H_sM_s^2 \notag\\
  &=
  2a\|PG_{T,\tau}e_k\|_2
  -2C_{T,\tau,P}a^2.
  \label{eq:proof-projected-lower-bound}
\end{align}
This is \eqref{eq:nonlinear-separation}. Since
$a\leq a_{\mathrm{lin}}$, its right side is at least
$a\|PG_{T,\tau}e_k\|_2>0$.

\medskip\noindent\textbf{Width and trilemma conclusion.}
Worst-case soundness over $\cC_{\mathrm{sm}}$ places both terminal states in
the same $\cX_T$. Hence
\begin{equation}
  \rho_P(\cX_T)
  \geq
  \frac12\|P\delta_T\|_2
  \geq
  a\|PG_{T,\tau}e_k\|_2-C_{T,\tau,P}a^2.
  \label{eq:proof-terminal-width}
\end{equation}
If $D_T>2\bar\rho$, \eqref{eq:proof-terminal-width} contradicts a reported
half-width cap $\bar\rho$. Together with Level A, this proves the stated
soundness, informativeness, and unrestricted-extrapolation alternative.
\end{proof}

\begin{corollary}[Consequences for informative tubes and certification]
\label{cor:planwise-consequences-supp}
Let $\cA$ be any admissible deterministic certifier.
\begin{enumerate}
  \item[(i)] Assume the Level-B hypotheses of
  Theorem~\ref{thm:trilemma}. In particular, the constructed residuals belong
  to $\cC_{\mathrm{sm}}$ for every $0<a\leq a_\star$, where $a_\star>0$,
  and $\|PG_{T,\tau}e_k\|_2>0$. If $\cA$ is worst-case sound over
  $\cC_{\mathrm{sm}}$ on $\pi$ and
  \begin{equation}
    \max_{0<a\leq a_\star}
    \left[a\|PG_{T,\tau}e_k\|_2-C_{T,\tau,P}a^2\right]
    >\bar\rho,
    \label{eq:corollary-threshold}
  \end{equation}
  then $\cA$ either abstains or is not $\bar\rho$-informative on $\pi$.

  \item[(ii)] Let $\cC$ be any residual class and suppose $\cA$ is worst-case
  sound over $\cC$ on $\pi$. If
  \begin{equation}
    \exists r\in\cC,\quad
    \exists t\in\{0,\ldots,T\}
    \quad\text{such that}\quad
    Px_t^r\in\mathcal O,
    \label{eq:corollary-unsafe-witness}
  \end{equation}
  then $\cA$ returns no safety certificate on $\pi$, and
  \begin{equation}
    \operatorname{Cert}^{\mathrm{wc}}_{\pi}(\cA)=0.
    \label{eq:corollary-zero-certification}
  \end{equation}
  Under the Level-B hypotheses, the construction supplies this witness whenever
  \begin{equation}
    \exists a\in(0,a_\star],\quad
    \exists s\in\{+,-\}
    \quad\text{such that}\quad
    Px_T^{(s)}(a)\in\mathcal O.
    \label{eq:constructed-unsafe-witness}
  \end{equation}
\end{enumerate}
\end{corollary}

\begin{proof}
For part~(i), Theorem~\ref{thm:trilemma} gives the terminal half-width lower
bound for every $a\in(0,a_\star]$. Taking the maximum over $a$ makes that
half-width exceed $\bar\rho$, so a worst-case-sound certifier either abstains or
violates $\bar\rho$-informativeness.

For part~(ii), soundness places $x_t^r$ in $\cX_t$, so
$Px_t^r\in P\cX_t\cap\mathcal O$, contradicting the safety-certificate
condition. Hence $\cA$ does not certify and
$\operatorname{Cert}^{\mathrm{wc}}_{\pi}(\cA)=0$.
\end{proof}

\section{Nominal-Map Enclosure}
\label{sec:nominal-enclosure}

Fix a propagation step $t$. The query set is the zonotope
$Z_t=\langle c_{Z,t},G_{Z,t}\rangle\subset\R^{d_z}$ with $d_z=n+m$, the stacked
state and applied control, and generator matrix
$G_{Z,t}=[\,g_{t,1}\ \cdots\ g_{t,\gamma_t}\,]$. Its interval hull is
$\mathrm{Box}(Z_t)=[\,c_{Z,t}-r_t,\;c_{Z,t}+r_t\,]$ with componentwise half-width
$r_t:=\sum_{j=1}^{\gamma_t}\lvert g_{t,j}\rvert=\lvert G_{Z,t}\rvert\mathbf 1$,
the absolute values taken componentwise. Every $z\in Z_t$ is
$z=c_{Z,t}+G_{Z,t}\xi$ with $\lVert\xi\rVert_\infty\leq1$, so
$\lvert z-c_{Z,t}\rvert\leq r_t$ componentwise and $Z_t\subseteq\mathrm{Box}(Z_t)$.

Let $J_t=D_zf_0(c_{Z,t})$. For each output coordinate $k$ let
$[\,H^-_{t,k},H^+_{t,k}\,]$ be an interval matrix that brackets the Hessian of
$f_{0,k}$ over the hull, meaning
$H^-_{t,k,ij}\leq\partial^2 f_{0,k}(\zeta)/\partial z_i\partial z_j\leq H^+_{t,k,ij}$
for every $\zeta\in\mathrm{Box}(Z_t)$. With
$B_{t,k,ij}=\max(\lvert H^-_{t,k,ij}\rvert,\lvert H^+_{t,k,ij}\rvert)$ the
zero-centered componentwise half-width is
\begin{equation}
  q_{t,k}=\tfrac12\,r_t^{\top}B_{t,k}\,r_t
         =\tfrac12\sum_{i,j}r_{t,i}\,B_{t,k,ij}\,r_{t,j}.
  \label{eq:qt-formula}
\end{equation}
The interval Hessian is constructed using monotone and branch-aware interval
extensions of the trigonometric and reciprocal terms.

\begin{proposition}[Nominal-map enclosure]
\label{prop:nominal-enclosure}
Suppose $f_0$ is twice continuously differentiable on $\mathrm{Box}(Z_t)$ and the
interval Hessian bounds $[\,H^-_{t,k},H^+_{t,k}\,]$ are finite. With $q_t$ from
Eq.~\eqref{eq:qt-formula},
\begin{equation}
  f_0(Z_t)\ \subseteq\ \big\langle f_0(c_{Z,t}),\,
   [\,J_tG_{Z,t}\quad\operatorname{diag}(q_t)\,]\big\rangle,
  \label{eq:nominal-enclosure}
\end{equation}
where the generator block $\operatorname{diag}(q_t)$ is centered at zero. If
$f_0$ is affine on $\mathrm{Box}(Z_t)$ then $q_t=0$ and the inclusion is the
exact image $\langle f_0(c_{Z,t}),J_tG_{Z,t}\rangle=f_0(Z_t)$. If the
interval-Hessian construction returns no finite bound on $\mathrm{Box}(Z_t)$, the
half-width is $+\infty$ and the certifier returns the domain abstention
$\abstainD$ at step $t$ under a distinct internal cause from a query that leaves
$\cZ_{\mathrm{cert}}$, and extends no tube past it.
\end{proposition}

\begin{proof}
Fix $z\in Z_t$ and an output coordinate $k$, and set $d=z-c_{Z,t}$. Since
$\mathrm{Box}(Z_t)$ is convex and contains $c_{Z,t}$ and $z$, the segment
$\zeta(\alpha)=c_{Z,t}+\alpha d$ for $\alpha\in[0,1]$ lies in
$\mathrm{Box}(Z_t)$, where $f_{0,k}$ is twice continuously differentiable.
Taylor's theorem with the Lagrange remainder gives some $\alpha_k\in(0,1)$ with
\begin{equation*}
  f_{0,k}(z)=f_{0,k}(c_{Z,t})+J_{t,k}\,d
   +\tfrac12\,d^{\top}\nabla^2 f_{0,k}\!\big(\zeta(\alpha_k)\big)\,d .
\end{equation*}
Because $\zeta(\alpha_k)\in\mathrm{Box}(Z_t)$, the remainder obeys
\begin{equation*}
  \begin{aligned}
  \Big|\tfrac12\,d^{\top}\nabla^2 f_{0,k}(\zeta(\alpha_k))\,d\Big|
   &\leq\tfrac12\sum_{i,j}\lvert d_i\rvert\,B_{t,k,ij}\,\lvert d_j\rvert\\
   &\leq\tfrac12\sum_{i,j}r_{t,i}\,B_{t,k,ij}\,r_{t,j}=q_{t,k},
  \end{aligned}
\end{equation*}
using $\lvert\partial^2 f_{0,k}(\zeta)/\partial z_i\partial z_j\rvert\leq B_{t,k,ij}$
on the hull and $\lvert d\rvert\leq r_t$ componentwise. Hence
$f_{0,k}(z)\in f_{0,k}(c_{Z,t})+J_{t,k}d+[-q_{t,k},q_{t,k}]$. Collecting the $n$
output coordinates, $f_0(z)-f_0(c_{Z,t})-J_td\in\operatorname{diag}(q_t)[-1,1]^n$,
and $d=G_{Z,t}\xi$ with $\lVert\xi\rVert_\infty\leq1$, so
$f_0(z)\in\langle f_0(c_{Z,t}),[J_tG_{Z,t}\ \operatorname{diag}(q_t)]\rangle$.
As $z\in Z_t$ was arbitrary, Eq.~\eqref{eq:nominal-enclosure} follows. When $f_0$
is affine the Hessian vanishes, so $B_{t,k}=0$ and $q_t=0$, and the second-order
term is absent, leaving the exact affine image.
\end{proof}

The certifier queries $Z_t$ only after the domain test places
$\mathrm{Box}(Z_t)\subseteq\cZ_{\mathrm{cert}}$, and $f_0$ is twice continuously
differentiable on $\cZ_{\mathrm{cert}}$, so the hypothesis of
Proposition~\ref{prop:nominal-enclosure} holds at every completed step. Because
$\mathrm{Box}(Z_t)$ is the convex hull of its vertices and $\cZ_{\mathrm{cert}}$
is convex, containment of every hull vertex implies containment of the full hull,
and the certifier verifies the inclusion by testing the finitely many vertices. A
query whose hull reaches a point where $f_0$ loses that smoothness, such as a steering
interval that meets the tangent pole, produces an infinite interval Hessian and
the domain abstention above, so a finite remainder is never formed on a region
where $f_0$ is not twice differentiable.

\subsection{Support-Function Evaluation of the Set-to-Sample Distance}
\label{sec:support-distance}

For $v\in\R^d$, the support function of a zonotope $Z=\langle c_Z,G_Z\rangle$ is
\begin{equation}
  h_Z(v):=\sup_{z\in Z}v^\top z=v^\top c_Z+\|G_Z^\top v\|_1.
  \label{eq:zonotope-support}
\end{equation}
The main-text set-to-sample distance
$D_{ik}(Z)=\sup_{z\in Z}\sum_{\ell=1}^{d_z}L_{k\ell}|z_\ell-z_{i,\ell}|$ admits an
exact center-generator form. With the query written as the zonotope
$Z=\langle c_Z,G_Z\rangle$ and, for each residual coordinate
$k\in\{1,\ldots,n\}$,
$W_k:=\operatorname{diag}(L_{k1},\ldots,L_{kd_z})\in\R_{\geq0}^{d_z\times d_z}$,
\begin{align}
  D_{ik}(Z)
  &=\sup_{z\in Z}\|W_k(z-z_i)\|_1\notag\\
  &=\max_{s\in\{-1,1\}^{d_z}}
  \left[s^\top W_k(c_Z-z_i)+\|G_Z^\top W_ks\|_1\right].
  \label{eq:zonotope-set-distance}
\end{align}
The first equality rewrites the componentwise Lipschitz distance as a weighted
$\ell_1$ norm. The second uses the $\ell_1$ dual representation and the zonotope
support function in Eq.~\eqref{eq:zonotope-support}. The implementation evaluates
the closed form directly in center-generator coordinates.

\section{One-Sided Falsifiability Statements}

\subsection{Comparison of Six Assumption Families}

\begin{table*}[t]
\centering
\scriptsize
\caption{Comparison of the operative assumptions used by six method families.
Condition (i) asks whether the assumption can be
directly contradicted using only the available one-step residual data.
Condition (ii) asks whether the content of a true declaration is sufficient to
support out-of-support plan-level validity.}
\label{tab:assumption-audit}
\begin{tabularx}{\textwidth}{@{}p{0.20\textwidth}p{0.26\textwidth}p{0.12\textwidth}p{0.12\textwidth}Y@{}}
\toprule
Family & Assumption in this comparison & (i) Directly refutable? &
(ii) Sufficient if true? & Basis \\
\midrule
Distributional exchangeability &
Deployment queries follow the calibration law &
No &
No &
It gives marginal coverage for random queries. Selected-plan worst-case
containment requires a plan-level premise. \\
\addlinespace
Parametric residual law &
The fitted residual box transfers to deployment &
No &
No &
Gaussian shape is rejectable. Gaussian support remains unbounded and cannot
provide worst-case containment. \\
\addlinespace
Local density or fitted smoothness &
Every plan query has adequate supported neighbors or a valid fitted scale &
No &
No &
The missing-neighbor condition occurs only in the unsampled deployment region. \\
\addlinespace
Linearity &
The residual is affine and the augmented calibration regressor has full column rank &
Yes &
Yes, under the stated excitation condition &
Affine-consistency emptiness gives direct refutation. Full-rank regressors make
the bounded-noise parameter set bounded. \\
\addlinespace
Set-membership containment without regularity &
Containing every calibration residual is sufficient &
No &
No &
The containment check constrains residuals at the sample points and leaves
escape between them unrestricted. \\
\addlinespace
Componentwise residual regularity &
A declared componentwise bound controls residual growth between sample points &
Yes, when data expose a violating pair &
Yes, conditional on the declaration being true &
The pairwise inequality gives a contradiction certificate, and the same true
regularity content supports extrapolation between observations. As a family over
all declared bounds, componentwise regularity strictly contains the affine
residuals. \\
\bottomrule
\end{tabularx}
\end{table*}

\begin{proposition}[Declarations satisfying both criteria in this comparison]
\label{prop:declaration-uniqueness}
Restrict attention to the six assumption families and their specific
instantiations in Table~\ref{tab:assumption-audit}. Restrict the available
evidence to the one-step residual data $\cD_N$ and noise bound $\sigma$.
Deployment trajectories, online labels, and deployment-distribution
information define richer information settings.

Exactly two families in this comparison satisfy both conditions.
\begin{enumerate}
  \item The declaration admits a data-only inconsistency certificate
  using the available one-step residual observations.
  \item If the declared statement is actually true, its content is sufficient
  to construct an out-of-support model or residual envelope that can support
  worst-case plan-level validity when combined with sound propagation.
\end{enumerate}
They are affine residual structure and componentwise residual regularity.
Let $d=n+m$, let $\bar z_i=(z_i^\top,1)^\top\in\R^{d+1}$, and define the
augmented regressor
\begin{equation}
  \bar Z_N
  :=
  \begin{bmatrix}
    \bar z_1^\top\\[-1mm]
    \vdots\\[-1mm]
    \bar z_N^\top
  \end{bmatrix}.
  \label{eq:augmented-affine-regressor}
\end{equation}
For affine structure, the inconsistency certificate is emptiness of
\begin{equation}
  \mathcal{A}(\cD_N,\sigma)
  :=
  \left\{
  (M,b):
  \begin{array}{l}
  |\widetilde r_{i,k}-(Mz_i+b)_k|\leq\sigma_k\\
  \text{for all }i,k
  \end{array}
  \right\}.
  \label{eq:affine-consistency-set}
\end{equation}
and condition (ii) additionally requires
\begin{equation}
  \operatorname{rank}(\bar Z_N)=d+1.
  \label{eq:affine-excitation}
\end{equation}
This excitation condition makes every nonempty bounded-noise affine-consistency
set $\mathcal A(\cD_N,\sigma)$ bounded.
For componentwise residual regularity, the certificate is an observed pair and
coordinate satisfying
  \begin{equation}
    |\widetilde r_{i,k}-\widetilde r_{j,k}|
    >
    \sum_{\ell}L_{k\ell}|z_{i,\ell}-z_{j,\ell}|
    +2\sigma_k.
    \label{eq:pairwise-contradiction}
  \end{equation}
Every affine residual $r(z)=Mz+b$ obeys
\begin{equation}
  |r_k(z)-r_k(z')|
  \leq
  \sum_{\ell}|M_{k\ell}|\,|z_\ell-z_\ell'|,
  \label{eq:affine-implies-regularity}
\end{equation}
so it belongs to the componentwise class with $L=|M|$. The inclusion is strict
on any domain containing a nontrivial line segment because the componentwise
class also contains nonlinear Lipschitz residuals. Componentwise residual regularity
is therefore the weaker of the two declarations in this comparison. Non-violation of
\eqref{eq:pairwise-contradiction} leaves both declarations unresolved outside
the sampled geometry.
\end{proposition}

\begin{proof}
The six row-wise outcomes are listed in
Table~\ref{tab:assumption-audit}. Emptiness of
$\mathcal A(\cD_N,\sigma)$ is a data-only contradiction certificate for affine
structure. Under \eqref{eq:affine-excitation}, the linear map from each affine
parameter row to its values on the sample points is injective. Bounded observations
and bounded noise therefore make every parameter row in
$\mathcal A(\cD_N,\sigma)$ bounded. If the residual is affine, its true
parameter belongs to this set. The image
\begin{equation}
  \mathcal E_{\mathrm{aff}}(z)
  :=
  \{Mz+b:(M,b)\in\mathcal A(\cD_N,\sigma)\}
  \label{eq:affine-residual-envelope}
\end{equation}
is a bounded residual envelope at every fixed query and contains
the true affine residual. Combining this envelope with sound propagation gives
condition (ii) for the affine row.

For componentwise regularity, the pairwise test
\eqref{eq:pairwise-contradiction} directly contradicts an exposed violation,
and a true declaration supplies the Lipschitz-cone envelope used by sound
propagation. Thus the two stated families satisfy both conditions under their
respective premises.

For the ordering, if $r(z)=Mz+b$, then the triangle inequality gives
\eqref{eq:affine-implies-regularity} with $L=|M|$. Conversely, let $v\neq0$ be
the direction of a nontrivial line segment contained in the domain. After an
irrelevant translation, $r(z)=\sin(v^\top z)e_1$ is componentwise Lipschitz but
is not affine on that segment. Thus affine structure is a strict subclass,
which proves the stated both-criteria claim within this comparison.
\end{proof}

\begin{remark}[Affine control example]
The linear-pair experiment yields a nonempty affine-consistency set and full
reachable-tube coverage on the linear target, together with an exposed
affine inconsistency on the nonlinear target. Without corridor support, wrong
declarations pass the pairwise test yet remain unsound over a band 0.092 wide on
the linear target and 0.850 wide on the nonlinear knee target, so the one-sided
test nearly suffices on the linear target but not once the residual class is
nonlinear. These outcomes illustrate the two
corresponding table entries. Proposition~\ref{prop:declaration-uniqueness} derives affine
sufficiency from the explicit excitation condition
\eqref{eq:affine-excitation}.
\end{remark}

\begin{proposition}[Finite-data non-verification]
\label{prop:nonverification}
Fix any finite residual dataset and any finite declared componentwise
Lipschitz bound $L$. Suppose an open ball in $\cZ$ contains no observed
sample point. There exist residual functions that agree at every observed sample point and
therefore produce exactly the same values in every observed pairwise test, yet
violate the declared Lipschitz bound inside that unobserved ball. Hence
\begin{equation}
  \begin{gathered}
  \text{no observed pairwise violation}\\
  \not\Longrightarrow\
  \text{the declaration holds on all plan queries}.
  \end{gathered}
  \label{eq:nonverification}
\end{equation}
Finite one-step data can therefore reject a declaration exposed by the sampled
geometry. Non-rejection leaves global truth unresolved.
\end{proposition}

\begin{proof}
Let $B$ be an open ball disjoint from the finite set of sample points and choose a
nonzero $\psi\in C_c^\infty(B)$. For any data-consistent baseline residual
$r_0$, define $r_a=r_0+a\psi e_k$. Every $r_a$ agrees with $r_0$ at all
sample points, so all observed pairwise tests are identical.

The declared matrix $L$ is finite. By increasing $a$, the variation of
$a\psi e_k$ between two points of $B$ exceeds the finite right side of the
declared componentwise inequality, after accounting for the fixed variation of
$r_0$. Thus some $r_a$ violates the declaration inside $B$ while every
observed test remains unchanged.
\end{proof}

Figure~\ref{fig:supp-adversarial} shows this adversarial residual family.

\begin{figure}[htbp]
  \centering
  \begin{subfigure}{0.85\columnwidth}
    \centering\includegraphics[width=\textwidth]{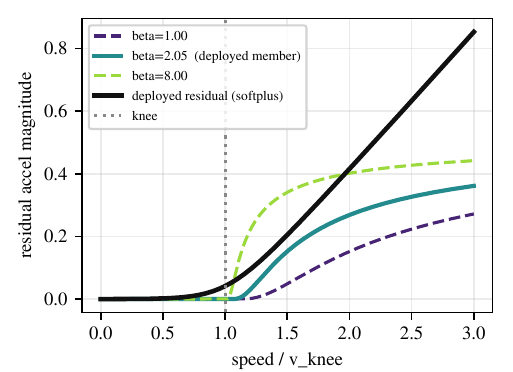}
    \caption{Escape profiles stay flat below the speed knee and steepen with the
    family parameter, where the deployed residual is a finite-parameter member.}
  \end{subfigure}
  \\[6pt]
  \begin{subfigure}{0.85\columnwidth}
    \centering\includegraphics[width=\textwidth]{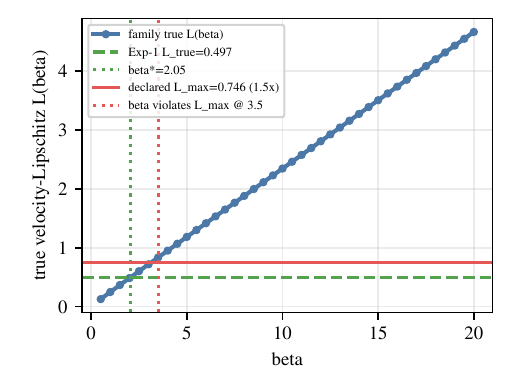}
    \caption{The true velocity Lipschitz constant grows without bound with the
    family parameter, so any finite declared bound is eventually violated once
    the escape steepens beyond the knee.}
  \end{subfigure}
  \caption{Adversarial residual family. Finite low-speed observations leave a plug-in
  slope estimate uncertified as a global bound.}
  \label{fig:supp-adversarial}
\end{figure}

\subsection{Slope-Estimator Declarations}

On the adversarial residual family, which is flat at every
sampled point and increases beyond the speed knee, the
LACKI~\citep{calliess2020lacki}, Strongin~\citep{strongin2000global}, and
Knuth-Chou-style slope estimators return 0.000,
0.000, and 0.000. These
implementations isolate the slope estimator used within LMTD-RRT
\citep{knuth2021learneddynamics,knuth2022correction}, whose full pipeline
certifies a feedback policy, whereas the present comparison applies only the
slope estimator to the fixed open-loop plan. Applying the Lipschitz declaration used by the certifier produces a
false certificate, while containment requires 31.710
times that declared bound, since identical sampled slopes leave residual
growth beyond the knee undetermined.

\section{CP-SLS-MPC Certificate-Budget Decomposition by Calibration Size}

We evaluate the published weighted-conformal, drift, and tube-budget equations
on the fixed candidate plans. The comparison retains the weighted-conformal
quantile, the total-variation drift, and the tube-coupled budget, and uses a
local covariance estimate, fixed response maps, and a fixed pre-execution
calibration set in place of the learned covariance model, the joint SLS-MPC
optimization, and online augmentation.

For horizon $T$, target miscoverage
$\alpha$, per-step level $\alpha_k=\alpha/T$, one-step calibration residual
$\varepsilon_k:=\varepsilon(x_k,u_k)$, and conformal set $E_k:=E(z_k,v_k)$, the
certificate of \citet{srinivasan2026cpslsmpc} reads
\begin{align}
  \Pr\!\left[\textstyle\bigcap_{k=1}^{T}\varepsilon_k\in E_k\right]
  &\geq 1-\sum_{k=1}^{T}\sigma_k,\notag\\
  \sigma_k&=\alpha_k+2\sum_i \widetilde w_i^{\,k}\,
    d_{\mathrm{TV}}(S^{i,k},S^{k,k})\notag\\
  &\quad+\gamma(R_k),\notag\\
  \gamma(R_k)&=2\widehat\varepsilon\,M(R_k),\notag\\
  \widetilde w_i^{\,k}&=\frac{w_i^{\,k}}{1+\sum_j w_j^{\,k}}.
  \label{eq:cpsls-budget}
\end{align}
The drift obeys $d_{\mathrm{TV}}(S^{i,k},S^{k,k})\leq
\widehat\varepsilon\,\lVert(z_k,v_k)-(x_i,u_i)\rVert_2$. Here $w_i^{\,k}$ are the
localized calibration weights, $\widetilde w_i^{\,k}$ their query-self-normalized
form, $d_{\mathrm{TV}}$ the total-variation distance between the residual laws
$S^{i,k}$ and $S^{k,k}$, $\widehat\varepsilon$ the estimated total-variation
Lipschitz drift constant, and $M(R_k)$ the maximum tube-axis length of the
reachable set $R_k$. The weighted split-conformal quantile $q_{1-\alpha_k}$
becomes infinite once the effective mass $\sum_j w_j^{\,k}$ falls below the
starvation floor $1/\alpha_k-1$, which the query self-mass in
$\widetilde w_i^{\,k}$ induces.

Two failure modes are distinct. Below the starvation floor the certificate radius
is infinite as above, while above the floor every per-step radius is finite yet
$\sum_k\sigma_k\geq1$ still holds, so the reported probability bound
$1-\sum_k\sigma_k$ is finite but vacuous. The certificate is non-vacuous only when
$\sum_k\sigma_k<1$.

Figure~\ref{fig:supp-cpslsmpc} reports the effective calibration mass and the
complete three-term budget decomposition across the tested calibration sizes.

\begin{figure*}[t]
  \centering
  \begin{subfigure}{0.48\textwidth}
    \centering\includegraphics[width=0.82\textwidth]{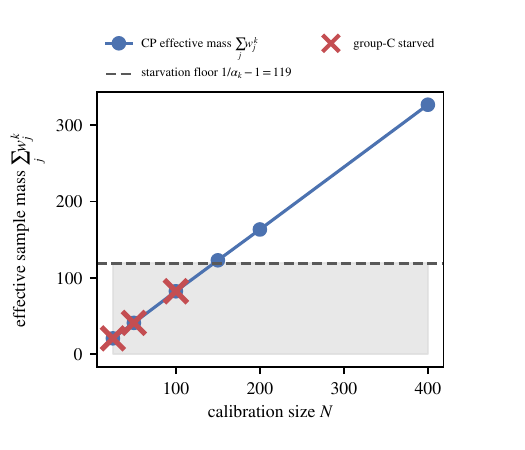}
    \caption{Weighted-conformal effective mass crossing the starvation floor.}
  \end{subfigure}\hfill
  \begin{subfigure}{0.48\textwidth}
    \centering\includegraphics[width=0.82\textwidth]{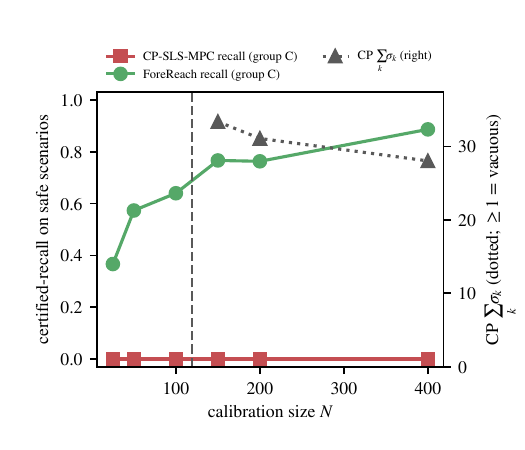}
    \caption{Certified recall on the sampled corridor across calibration size.}
  \end{subfigure}
  \\[4pt]
  \begin{subfigure}{0.48\textwidth}
    \centering\includegraphics[width=0.82\textwidth]{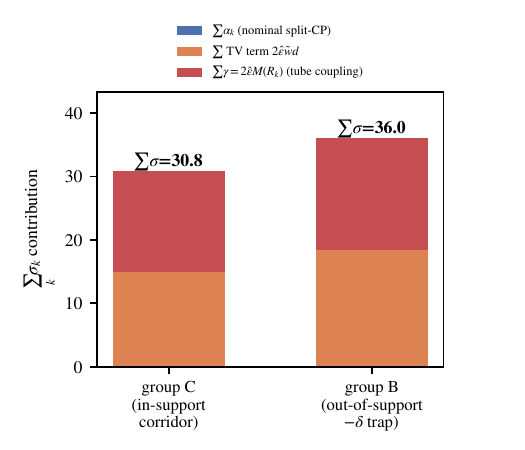}
    \caption{Three-term budget decomposition at $N=200$.}
  \end{subfigure}\hfill
  \begin{subfigure}{0.48\textwidth}
    \centering\includegraphics[width=0.82\textwidth]{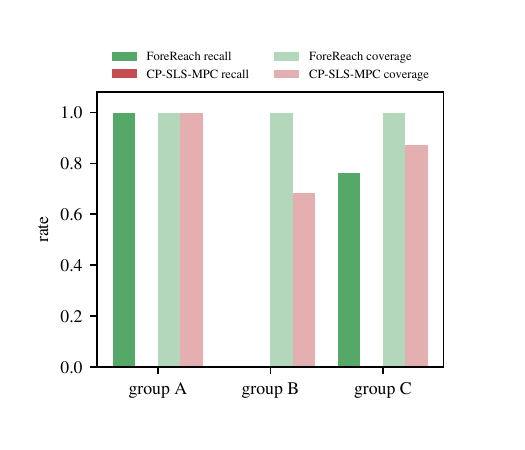}
    \caption{Per-group recall and coverage at $N=200$.}
  \end{subfigure}
  \caption{Evaluation of the published CP-SLS-MPC certificate-budget equations on
  the fixed candidate plans. Group A is the in-support low-demand region, group B
  the unsupported plans, and group C the corridor-supported plans.}
  \label{fig:supp-cpslsmpc}
\end{figure*}

\begin{table}[htbp]
\centering
\small
\caption{CP-SLS-MPC all-$N$ budget decomposition for the 6D dynamic bicycle. The reported total is the median of per-scenario total budgets. Component medians are descriptive and non-additive.}
\label{tab:cpslsmpc-all-n}
{\scriptsize\setlength{\tabcolsep}{4pt}%
\begin{tabular}{@{}rrrrrrr@{}}
\toprule
$N$ & starved & nonvacuous & $\alpha$ & TV & tube & median(total) \\
\midrule
25 & 100.0\% & 0 & -- & -- & -- & $+\infty$ \\
50 & 100.0\% & 0 & -- & -- & -- & $+\infty$ \\
100 & 100.0\% & 0 & -- & -- & -- & $+\infty$ \\
150 & 14.7\% & 0 & 0.10 & 15.11 & 15.15 & 30.38 \\
200 & 0.0\% & 0 & 0.10 & 14.84 & 15.88 & 31.08 \\
400 & 0.0\% & 0 & 0.10 & 14.78 & 15.76 & 30.55 \\
\bottomrule
\end{tabular}}
\end{table}

\section{Experimental Details and Reproducibility}

The accompanying archive contains code, configurations, and aggregate records that support checking the reported numerical results.

\subsection{Certified Point-Mass Regularity Declaration}
\label{sec:pointmass-certified-bound}

The point-mass implementation defines the residual directly in the
discrete update. Let $v=(v_x,v_y)$,
$s=(\|v\|_2^2+\delta_s)^{1/2}$ with $\delta_s=10^{-12}$, and
\begin{align}
  \ell(s)&=\beta_{\mathrm{pm}}\log\!\left(1+
    \exp\!\left(\frac{s-v_{\rm knee}}{\beta_{\mathrm{pm}}}\right)\right),
  &g(s)&=\frac{\ell(s)}{s+\epsilon},\notag\\
  M&=\begin{bmatrix}-\gamma&-\kappa\\ \kappa&-\gamma\end{bmatrix},\notag\\
  a_r(v)&=g(s)Mv. \label{eq:pointmass-residual-accel}
\end{align}
The state residual is
\begin{equation}
  r(x,u)=
  \begin{bmatrix}\tfrac12\Delta t^2a_r(v)\\ \Delta t\,a_r(v)\end{bmatrix}.
  \label{eq:pointmass-discrete-residual}
\end{equation}
Equations~\eqref{eq:pointmass-residual-accel} to
\eqref{eq:pointmass-discrete-residual} use $\gamma=0.4$, $\kappa=0.3$,
$v_{\rm knee}=0.9$, softplus temperature $\beta_{\mathrm{pm}}=0.14$, regularizer
$\epsilon=0.15$, and step $\Delta t=0.15$. They show analytically that
$\partial r_k/\partial p_x=\partial r_k/\partial p_y=0$ and
$\partial r_k/\partial u_x=\partial r_k/\partial u_y=0$ for every residual
coordinate.

For $q=\|v\|_2>0$, write $v=qn$ with $\|n\|_2=1$. Define
\begin{align}
  p(s)&=\left(1+\exp\!\left(-\frac{s-v_{\rm knee}}{\beta_{\mathrm{pm}}}\right)\right)^{-1},\notag\\
  \nu(s)&=(s+\epsilon)p(s)-\ell(s),\notag\\
  h(s)&=\frac{q^2}{s}\frac{\nu(s)}{(s+\epsilon)^2}.
  \label{eq:pointmass-radial-terms}
\end{align}
Direct differentiation gives
\begin{equation}
  D_va_r(v)=g(s)M+h(s)(Mn)n^\top.
  \label{eq:pointmass-residual-jacobian}
\end{equation}
The same formula holds at $q=0$ by continuity. For the fixed parameter values,
$\nu$ is nonnegative at the lower endpoint and
$\nu'(s)=(s+\epsilon)p(s)(1-p(s))/\beta_{\mathrm{pm}}\geq0$, hence $g$ and $h$ are
nonnegative on the certification interval. Let
$m=(\gamma^2+\kappa^2)^{1/2}$. Maximizing the quadratic angular factors in
Eq.~\eqref{eq:pointmass-residual-jacobian} gives the exact fixed-radius
componentwise extrema
\begin{align}
  d_{\rm diag}(s)&=\gamma g(s)+\frac{\gamma+m}{2}h(s),\notag\\
  d_{\rm cross}(s)&=\kappa g(s)+\frac{\kappa+m}{2}h(s).
  \label{eq:pointmass-angle-maxima}
\end{align}

The certified velocity domain is
$0\leq\|v\|_2\leq5.4$. Position and action extents do
not affect this derivative bound because their columns are structural zeros.
The remaining radial maximization uses
20000 cells at 80-decimal precision. On a
cell $[s_a,s_b]$, monotonicity gives $g(s)\leq g(s_b)$ and
$\nu(s)\leq\nu(s_b)$. The factor
\begin{equation}
  w(s)=\frac{s^2-\delta_s}{s(s+\epsilon)^2},\qquad h(s)=\nu(s)w(s),
\end{equation}
attains its cell maximum at an endpoint or at the unique positive root of
$s^3-\epsilon s^2-3\delta_s s-\epsilon\delta_s=0$. Evaluating these candidates
and rounding the reported decimal upward encloses both functions in
Eq.~\eqref{eq:pointmass-angle-maxima}.

The resulting componentwise matrix is
\begin{equation}
L^{\rm cert}=\begin{bmatrix}
0&0&0.004609514&0.003793347&0&0\\
0&0&0.003793347&0.004609514&0&0\\
0&0&0.061460184&0.050577964&0&0\\
0&0&0.050577964&0.061460184&0&0
\end{bmatrix}.
\label{eq:pointmass-certified-matrix}
\end{equation}
The comparison below gives each unique certified entry and the smallest
corresponding deployed entry across the symmetric coordinates. The deployed
matrix is the base declared Lipschitz bound scaled by the deployment inflation
factor $1.5$, whereas $L^{\rm cert}$ is certified independently over the domain,
so the entrywise bound $L^{\rm cert}\leq$ deployed confirms that the deployed
declaration bounds the residual regularity.

\begin{center}
\centering
\small
\begin{tabular}{@{}lcc@{}}
\toprule
Entry class & $L^{\rm cert}$ & Minimum deployed entry \\
\midrule
Position diagonal & 0.004609514 & 0.006914008 \\
Position cross & 0.003793347 & 0.005689541 \\
Velocity diagonal & 0.061460184 & 0.092186775 \\
Velocity cross & 0.050577964 & 0.075860540 \\
\bottomrule
\end{tabular}
\end{center}

The minimum nonzero entrywise margin over
Eq.~\eqref{eq:pointmass-certified-matrix} is 0.001896193.
The deployed four-dimensional matrix is therefore a certified
Lipschitz bound on the stated velocity domain.

\subsection{Dynamic-Bicycle Certified Regularity Declaration}
\label{sec:bicycle-candidate-bound}

The six-dimensional state is
$x=(x,y,\theta,v,a,\delta)$ with units
$(\mathrm{m},\mathrm{m},\mathrm{rad},\mathrm{m/s},
\mathrm{m/s^2},\mathrm{rad})$. The control is $(j,\omega)$ in
$(\mathrm{m/s^3},\mathrm{rad/s})$. The residual has the same units as the
discrete next-state coordinates. Its only analytic nonzero output is the
heading residual
\begin{equation}
  \begin{split}
  r_\theta(v,\delta)={}&-\frac{\Delta t\,\kappa_{us}}{L_{\mathrm{wb}}}\,v\tan\delta\;\times{}\\
  &\exp\!\left(\frac{-1}{\beta_{\mathrm{bic}}\left(\lvert v^2\tan\delta/L_{\mathrm{wb}}\rvert-\mu g\right)}\right)
  \end{split}
  \label{eq:bicycle-residual-heading}
\end{equation}
above the lateral-saturation knee and zero below it, with $\Delta t=0.1$,
$\kappa_{us}=0.32$, $L_{\mathrm{wb}}=2.5$, gate temperature
$\beta_{\mathrm{bic}}=0.1$, friction $\mu=0.9$, and $g=9.81$. The residual reads the state only through
$v$ and $\delta$, so every other output coordinate and every control derivative
is a structural zero.

Because $r_\theta$ is odd in $\delta$, the magnitudes
$\lvert\partial r_\theta/\partial v\rvert$ and
$\lvert\partial r_\theta/\partial\delta\rvert$ are even in $\delta$, so the
certified maximization runs on the positive quadrant
$0\leq v\leq25.00$ and
$0\leq\delta\leq0.48$. Writing
$E=\beta_{\mathrm{bic}}\left(v^2\tan\delta/L_{\mathrm{wb}}-\mu g\right)$ for the gate argument and
$\psi_{\mathrm{bic}}(E)=e^{-1/E}/E^2$, the partials are
\begin{align}
  \left\lvert\frac{\partial r_\theta}{\partial v}\right\rvert
   &=c_0\tan\delta\,e^{-1/E}
     +\frac{2c_0\beta_{\mathrm{bic}}}{L_{\mathrm{wb}}}\,v^2\tan^2\!\delta\,\psi_{\mathrm{bic}}(E),\notag\\
  \left\lvert\frac{\partial r_\theta}{\partial\delta}\right\rvert
   &=c_0\,v\sec^2\!\delta\,e^{-1/E}
     +\frac{c_0\beta_{\mathrm{bic}}}{L_{\mathrm{wb}}}\,v^3\sec^2\!\delta\,\tan\delta\,\psi_{\mathrm{bic}}(E),
  \label{eq:bicycle-certified-partials}
\end{align}
with $c_0=\Delta t\,\kappa_{us}/L_{\mathrm{wb}}$. On the positive quadrant $E$ increases in
both $v$ and $\delta$, the gate $e^{-1/E}$ increases in $E$, and $\psi_{\mathrm{bic}}$ is
unimodal with global maximum $4e^{-2}$ at $E=\tfrac12$. On a grid cell every
factor is therefore bounded by its cell supremum, taken at the upper speed and
steering corner for the algebraic factors and at the interior peak $E=\tfrac12$
for $\psi_{\mathrm{bic}}$ whenever the cell contains it, so the per-cell value dominates the true derivative at every point of the cell.
A grid of $16000$ by $8000$
cells over the operating box yields the certified matrix whose two nonzero
entries are
\begin{center}
\begin{tabular}{@{}lcc@{}}
\toprule
Derivative & Certified bound & Deployed entry \\
\midrule
$|\partial r_\theta/\partial v|$ & 0.011965171 & 0.017926410 \\
$|\partial r_\theta/\partial\delta|$ & 0.407688551 & 0.609638820 \\
\bottomrule
\end{tabular}
\end{center}
All other entries are structural zeros. The certified $v$-derivative maximizer
lies at the upper steering edge with an interior speed near the gate peak, and
the certified $\delta$-derivative maximizer lies at the upper speed and steering
corner. The minimum nonzero entrywise margin between the deployed and certified
matrices is $0.005961239$, so the deployed six-dimensional
matrix is a certified Lipschitz bound on the stated speed and steering
domain. As in the point-mass declaration, the deployed matrix is the base
declared Lipschitz bound scaled by the factor $1.5$, and the certified matrix
satisfies $L^{\rm cert}\leq$ deployed entrywise.

\subsection{Dynamic-Bicycle Protocol}

The dynamic-bicycle study uses a kinematic-bicycle nominal model and a
dynamic-bicycle target with lateral-force saturation, with the state, control,
step, horizon, calibration sizes, seeds, and scenario counts listed in
Table~\ref{tab:final-configuration-index}. Group A is an in-support, low-demand
control region, group B is an out-of-support, high lateral-demand region in the
uncalibrated negative-steering direction that the main text calls the unsupported
plans, and group C is the calibrated high lateral-demand positive-steering corridor
used for the informative-recall comparison that the main text calls the
corridor-supported plans. The sub-Gaussian channel adds independent componentwise
noise of scale $4.16\times10^{-6}$, five percent of the median one-step residual
amplitude $8.33\times10^{-5}$, so the noise term is a median fraction
$3.4\times10^{-4}$ of the residual-envelope width along the certified corridor and a
negligible part of the propagated tube width.

\subsection{Pre-Execution Action-Chunk Protocol}

The action-chunk protocol changes only the source of the fixed open-loop plan,
where a learned policy emits a 20-step control block that the six-state certifier
of the previous experiments receives unchanged. The nominal and target dynamics,
residual declaration, sample placement, noise channels, obstacle test, and
certification-domain check are inherited from the six-state experiment. An iLQR
expert tracks a feasible nominal reference, straight for group A and arced for
groups B and C, and a behavior-cloning policy with two width-64 hidden layers maps
a 12-dimensional policy input to the 40-dimensional action block.

The action-chunk grid uses ten seeds, $N\in\{25,50,100,150,200,400\}$, both noise
channels, and 250 plans in each of groups A to C. Table~\ref{tab:action-chunk}
reports the certificate rate and certified recall for the fixed-grid and
policy-corridor layouts. Every seed reuses the same 750 policy blocks across
sample counts, channels, and observation layouts, and each target-system
evaluation uses 200 initial-state realizations. The support statistic is
\begin{equation}
d=\max_t\min_i\left\|(x_t-x_i)/[40,40,\pi,25,3,0.5]\right\|_2,
\end{equation}
computed independently of the certificate. Over 45000 records the Spearman
correlation between $d$ and abstention is 0.697, the top-minus-bottom
distance-quartile abstention gap is 0.908, and across all 240 cells there are no
issued-certificate undercoverage events and no false certificates under valid
declarations.

Figure~\ref{fig:supp-firstabstain} plots the first-abstention-step distribution
across both sample layouts.

\begin{figure}[ht]
  \centering
  \includegraphics[width=0.72\columnwidth]{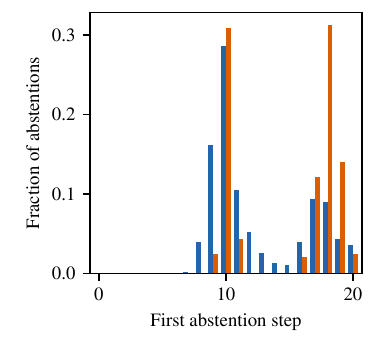}
  \\[3pt]
  \includegraphics[width=0.55\columnwidth]{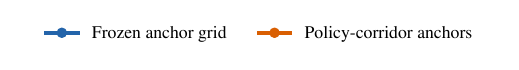}
  \caption{First-abstention-step distribution, a diagnostic of where
  certification first stops along the propagated action block.}
  \label{fig:supp-firstabstain}
\end{figure}

\begin{table}[htbp]
\centering
\scriptsize
\setlength{\tabcolsep}{4pt}
\caption{Action-chunk certification over ten seeds in the noiseless channel. Entries are mean $\pm$ sample standard deviation. Coverage is $1.000\pm0.000$ in every cell, and the sub-Gaussian channel tracks the noiseless channel within one point, both retained in full in the released code.}
\label{tab:action-chunk}
\begin{tabular}{@{}lrrr@{}}
\toprule
Sample layout & $N$ & Certificate rate & Certified recall \\
\midrule
Fixed grid & 25 & 0.237 $\pm$ 0.02433 & 0.356 $\pm$ 0.03650 \\
Fixed grid & 50 & 0.343 $\pm$ 0.02391 & 0.514 $\pm$ 0.03586 \\
Fixed grid & 100 & 0.401 $\pm$ 0.01584 & 0.601 $\pm$ 0.02376 \\
Fixed grid & 150 & 0.421 $\pm$ 0.01696 & 0.632 $\pm$ 0.02545 \\
Fixed grid & 200 & 0.437 $\pm$ 0.01481 & 0.655 $\pm$ 0.02222 \\
Fixed grid & 400 & 0.463 $\pm$ 0.00790 & 0.695 $\pm$ 0.01185 \\
\addlinespace
Policy corridor & 25 & 0.347 $\pm$ 0.02224 & 0.521 $\pm$ 0.03336 \\
Policy corridor & 50 & 0.463 $\pm$ 0.02473 & 0.694 $\pm$ 0.03710 \\
Policy corridor & 100 & 0.551 $\pm$ 0.02319 & 0.826 $\pm$ 0.03478 \\
Policy corridor & 150 & 0.596 $\pm$ 0.02067 & 0.893 $\pm$ 0.03100 \\
Policy corridor & 200 & 0.614 $\pm$ 0.01856 & 0.921 $\pm$ 0.02784 \\
Policy corridor & 400 & 0.647 $\pm$ 0.01149 & 0.971 $\pm$ 0.01723 \\
\bottomrule
\end{tabular}
\end{table}

\subsection{Pairwise Consistency Test Rejection and Coverage Denominators}

The pairwise consistency test runs once for each domain, calibration size,
observation channel, and seed before any plan in that cell is propagated. A
propagated plan then receives exactly one of four outcomes, a pairwise rejection
$\abstainF$ with no tube, a domain exit $\abstainD$ with only a prefix tube, a
safety-test abstention $\abstainS$ after a complete tube, or a certificate after
a complete tube. Only $\abstainS$ and certificate outcomes produce a complete
tube, so $\abstainF$ and $\abstainD$ leave the conditional-coverage denominator,
prefix containment is reported separately, and the joint containment uses no
imputed denominator. Across the 140 consistency-test decisions there are 0
rejections, so every seed-level declaration enters propagation, while a plan that
passes the test can still leave the certification domain.

\subsection{Certification-Domain Check}

The certifier checks every query zonotope against the supplied certification
domain before evaluating the next residual envelope. Across
186000 queries it returns
\(\abstainD\) on 39807 in-loop exits, with no exit after the horizon completes.
These exits concentrate on the out-of-support group-B plans, while the
corridor-supported group-C plans reach the safety test, and the full per-cell cause
rates for every noiseless and sub-Gaussian cell are retained in the provided code.

\subsection{Reachable-Set Representation}

Under identical scenarios, residual observations, declarations, noise bounds, and
nonlinear-remainder rules, the paired box and zonotope representations isolate the
effect of preserving generator correlations. On the 4D point mass the projected-area
ratio is 1.000, so the box already suffices, while on the 6D bicycle the zonotope area
is 24 to 58 percent smaller, reaching a ratio of 0.421 on the corridor-supported group
C at $N=400$, where it certifies 0.492 of the plans against 0.415 for the box. The full
per-cell areas and safe rates are retained in the provided code.

The zonotope representation controls generator growth with the CORA-aligned Girard
order reduction \citep{althoff2015cora}. After each propagation step the operator
$\operatorname{red}$ caps the zonotope order at 5, so the reduced representation
keeps at most five generators per state dimension under the implementation order
convention. It retains the highest-scored generators and replaces the discarded
columns $g_j$, indexed by the discard set $\mathcal J_t$, with the diagonal outer
bound $\operatorname{diag}(\sum_{j\in\mathcal J_t}|g_j|)$, so that
$\widehat{\cX}_{t+1}\subseteq\operatorname{red}(\widehat{\cX}_{t+1})$ and the
containment premise of the main-text containment theorem holds.

The obstacle test runs directly on the zonotope, using its exact projected geometry.
For each projected slice $P\cX_t$ and each circular obstacle, the implementation
computes the exact distance from the projected zonotope polygon to the obstacle
center and certifies avoidance only when a support-function dual lower bound on
that distance, which closes its duality gap, exceeds the obstacle radius.
This realizes the sufficient condition $R_t<m_t$ of the main text without forming
$R_t$, so a certificate implies that the projected reachable set is disjoint from
the obstacle.

\subsection{Certificate-Budget and Envelope Sensitivity}

Under four covariance plugins for the CP-SLS-MPC certificate equations, evaluated across
the tested calibration sizes, every tested row is starved or
vacuous, and the complete sweep is retained in the provided code. The scalar-envelope
ablation, which collapses the componentwise residual bound to a single scalar, drives
every plan out of the declared domain, so the certificate rate is zero and
complete-trajectory coverage has an empty denominator. A sub-Gaussian noise sweep
on the six-state corridor at $N=400$ lowers the corridor-supported certified recall
from 0.983 at zero noise to 0.867 when the componentwise noise scale reaches 3.3 times
the median $\ell_\infty$ residual amplitude across the observations.

\subsection{Nominal-Model Degradation and Norm Ratios}

Here $L_f$ and $L_r$ are the componentwise Lipschitz matrices of the full
dynamics and of the residual over $\cZ_{\mathrm{cert}}$, each entry the supremum
of the corresponding absolute partial derivative, and $\|\cdot\|_F$ is the
Frobenius norm. The full-dynamics to residual Lipschitz-norm ratio
$\|L_f\|_F/\|L_r\|_F$ is 18.481 in the point-mass domain and
11.368 in the bicycle, so the benefit of the nominal model is
domain dependent. Under the degradation scan, the point-mass domain tolerates
16.638\% nominal degradation before corridor-supported certified
recall crosses the common reference level of $0.5$, while the bicycle tolerates
1.386\%.

\subsection{Parameter-Box Declaration}

Table~\ref{tab:parameter-box} reports the study that declares the
residual Jacobian bound uniformly over a design-parameter box. In
this synthetic study the nominal and target parameters coincide, the parameter
boxes are fixed before the bounds are computed, and the certifier reads only
their endpoints. As the box widens the slack ratio
$\|L^\Theta\|/\|L^{\mathrm{cert}}\|$ rises and the corridor-supported certified recall moves
from 0.993 at the $\pm5\%$ box to
0.870 at the $\pm20\%$ box, while the unsupported-plan false certificate count stays zero and the completeness rate stays near
one.

\begin{table}[htbp]
\centering
\small
\caption{Parameter-box demonstration on the six-state corridor at $N=400$, noiseless channel. Each tier declares the residual Jacobian bound uniformly over a design-parameter box and reads no target parameter value. The slack ratio is $\|L^\Theta\|/\|L^{\mathrm{cert}}\|$. The group-B false certificate count stays zero.}
\label{tab:parameter-box}
{\scriptsize\setlength{\tabcolsep}{4pt}%
\begin{tabular}{@{}lccccc@{}}
\toprule
Box & Slack ratio & $p_{\mathrm{complete}}$ & Joint & Certified recall & False cert. \\
\midrule
$\pm5\%$ & 1.177 & 1.000 & 1.000 & 0.993 & 0 \\
$\pm10\%$ & 1.606 & 1.000 & 1.000 & 0.975 & 0 \\
$\pm20\%$ & 2.838 & 0.997 & 0.997 & 0.870 & 0 \\
\bottomrule
\end{tabular}}
\end{table}

\subsection{Per-Seed Uncertainty}

Table~\ref{tab:per-seed-uncertainty} reports the validity and
informativeness rates with their seed-level dispersion for every method in the
noiseless channel, with the baselines shown at the representative calibration
sizes. The complete per-seed values and the sub-Gaussian channel are retained
in the provided code. Across the ten seeds our method holds Group-B coverage at
one with zero dispersion at every calibration size, so its validity does not
change with the seed draw, and its Group-C certified recall rises with the
calibration size while its standard deviation shrinks as more corridor
observations enter the envelope. The calibration baselines instead show wide
seed-level dispersion in Group-B coverage, which reflects that their validity
depends on whether the finite sample happens to cover the queried region.

\begin{table}[htbp]
\centering
\scriptsize
\setlength{\tabcolsep}{4pt}
\caption{Per-method uncertainty summary in the noiseless channel. Each entry is mean $\pm$ sample standard deviation over the ten fixed seeds. ForeReach is shown at all calibration sizes and uses the zonotope representation, and the seven baselines are shown at the representative sizes $N\in\{25,100,400\}$. The ForeReach sub-Gaussian rows match its noiseless rows because the sub-Gaussian noise scale is negligible for the group-C recall in this domain, so they are retained in full in the released code together with every baseline size.}\label{tab:per-seed-uncertainty}
\begin{tabular}{@{}lrcc@{}}
\toprule
Method & $N$ & Group-B coverage & Group-C certified recall \\
\midrule
ForeReach & 25 & $1.0000\pm 0.0000$ & $0.4899\pm 0.1251$ \\
ForeReach & 50 & $1.0000\pm 0.0000$ & $0.7850\pm 0.0857$ \\
ForeReach & 100 & $1.0000\pm 0.0000$ & $0.8817\pm 0.0638$ \\
ForeReach & 150 & $1.0000\pm 0.0000$ & $0.9128\pm 0.0540$ \\
ForeReach & 200 & $1.0000\pm 0.0000$ & $0.9496\pm 0.0262$ \\
ForeReach & 400 & $1.0000\pm 0.0000$ & $0.9832\pm 0.0198$ \\
\addlinespace
Global calibration & 25 & $0.4216\pm 0.2214$ & $0.0656\pm 0.1740$ \\
Global calibration & 100 & $0.1625\pm 0.1006$ & $0.4072\pm 0.4400$ \\
Global calibration & 400 & $0.0448\pm 0.0450$ & $0.9800\pm 0.0425$ \\
Regional calibration & 25 & $0.4216\pm 0.2214$ & $0.0656\pm 0.1740$ \\
Regional calibration & 100 & $0.8205\pm 0.2882$ & $0.1184\pm 0.3115$ \\
Regional calibration & 400 & $0.9015\pm 0.0491$ & $0.0696\pm 0.0300$ \\
Plug-in Gaussian & 25 & $0.1459\pm 0.1164$ & $0.9785\pm 0.0457$ \\
Plug-in Gaussian & 100 & $0.1773\pm 0.0336$ & $1.0000\pm 0.0000$ \\
Plug-in Gaussian & 400 & $0.1738\pm 0.0403$ & $1.0000\pm 0.0000$ \\
GP ($k=2$) & 25 & $0.0654\pm 0.0737$ & $0.9640\pm 0.0443$ \\
GP ($k=2$) & 100 & $0.0442\pm 0.0329$ & $1.0000\pm 0.0000$ \\
GP ($k=2$) & 400 & $0.0145\pm 0.0213$ & $1.0000\pm 0.0000$ \\
GP ($k=3$) & 25 & $0.2156\pm 0.1738$ & $0.9297\pm 0.0633$ \\
GP ($k=3$) & 100 & $0.1594\pm 0.0471$ & $1.0000\pm 0.0000$ \\
GP ($k=3$) & 400 & $0.1018\pm 0.0625$ & $1.0000\pm 0.0000$ \\
Alanwar Alg.~6 & 25 & $1.0000\pm 0.0000$ & $0.0000\pm 0.0000$ \\
Alanwar Alg.~6 & 100 & $1.0000\pm 0.0000$ & $0.0000\pm 0.0000$ \\
Alanwar Alg.~6 & 400 & $1.0000\pm 0.0000$ & $0.0000\pm 0.0000$ \\
CP-SLS-MPC equations & 25 & $0.0000\pm 0.0000$ & $0.0000\pm 0.0000$ \\
CP-SLS-MPC equations & 100 & $0.0000\pm 0.0000$ & $0.0000\pm 0.0000$ \\
CP-SLS-MPC equations & 400 & $0.6398\pm 0.0544$ & $0.0000\pm 0.0000$ \\
\bottomrule
\end{tabular}
\end{table}

\subsection{Configurations}

The configuration index below covers every experiment reported in the
manuscript and supplement, with full Lipschitz bounds and scenario bounds in the
provided code. Every experiment shares the zonotope representation at maximum
order five, the in-loop query-set domain check, the noiseless and sub-Gaussian
channels, and the ten fixed seeds, so each run is deterministic and reproduces
the reported numbers when it is replayed from the released code. The remaining
rows record the step size, horizon, calibration sizes, and scenario counts that
distinguish the four-dimensional point mass, the six-dimensional dynamic
bicycle, the action-chunk certification, the nominal-model value scans, the
noise-floor analyses, and the sensitivity study.

\begin{table}[H]
\centering
\scriptsize
\setlength{\tabcolsep}{3pt}
\caption{Configurations for experiments reported in the manuscript and supplement. Full residual bounds and scenario bounds remain in the released code.}\label{tab:final-configuration-index}
\begin{tabular}{@{}p{0.18\columnwidth}p{0.78\columnwidth}@{}}
\toprule
Experiment & Configuration \\
\midrule
Common protocol & Zonotope representation at maximum order 5, an in-loop query-set domain check, the noiseless and sub-Gaussian channels, componentwise Gaussian noise of standard deviation $\sigma=0.05\,\mathrm{median}_z\lVert r(z)\rVert_\infty$, and ten fixed seeds \\
4D point mass & Step 0.15, horizon 12, and 100 samples over groups A, B, and C, with 600 scenarios per group and seed and 300 truth rollouts per scenario \\
6D dynamic bicycle & Step 0.1, horizon 12, and sample sizes 25, 50, 100, 150, 200, and 400 over groups A, B, and C, with 250 scenarios per group and seed and 200 truth rollouts per scenario \\
Action-chunk certification & Horizon 20, the fixed-grid and corridor layouts, sample sizes 25 through 400, and both channels \\
Nominal-model value scans & 49 cells in 4D and 60 cells in 6D, both under the same implementation, channels, and seeds \\
Noise-floor analyses & A fixed-geometry $\sigma$ sweep in both domains and the two-system numerical noise-floor witness \\
Sensitivity analyses & Paired box and zonotope representation rows, four CP-SLS-MPC covariance plugins, and the scalar residual-envelope ablation \\
\bottomrule
\end{tabular}
\end{table}

\end{document}